\documentclass{article}

\usepackage{microtype}
\usepackage{graphicx}
\usepackage{subfigure}
\usepackage{booktabs} % for professional tables

\usepackage{hyperref}

\usepackage[accepted]{icml2020}

\usepackage[utf8]{inputenc} % allow utf-8 input
\usepackage[T1]{fontenc}    % use 8-bit T1 fonts
\usepackage{url}            % simple URL typesetting
\usepackage{amsfonts}       % blackboard math symbols
\usepackage{nicefrac}       % compact symbols for 1/2, etc.
\usepackage{amsmath}

\usepackage{amsmath,amsfonts,bm}

\def\eqref#1{equation~\ref{#1}}
\def\1{\bm{1}}

\DeclareMathAlphabet{\mathsfit}{\encodingdefault}{\sfdefault}{m}{sl}
\SetMathAlphabet{\mathsfit}{bold}{\encodingdefault}{\sfdefault}{bx}{n}

\usepackage{amsthm}

\usepackage{comment}
\usepackage{multirow}

\usepackage{wrapfig}

\usepackage{footnote}
\usepackage{breqn}
\makesavenoteenv{tabular}
\makesavenoteenv{table}
\makeatletter
\newcommand\footnoteref[1]{\protected@xdef\@thefnmark{\ref{#1}}\@footnotemark}
\makeatother

\newtheorem{theorem}{Theorem}
\newtheorem{definition}{Definition}
\newtheorem{lemma}{Lemma}

\newtheorem{proposition}{Proposition}
\graphicspath{{img/}}

\icmltitlerunning{On Layer Normalization in the Transformer Architecture}

\begin{document}

\twocolumn[
\icmltitle{On Layer Normalization in the Transformer Architecture}



\icmlsetsymbol{equal}{*}

\begin{icmlauthorlist}
\icmlauthor{Ruibin Xiong\textsuperscript{\dag}\textsuperscript{*}}{cas,ucas}
\icmlauthor{Yunchang Yang\textsuperscript{*}}{pku-bigdata}
\icmlauthor{Di He}{pku-moe,ms}
\icmlauthor{Kai Zheng}{pku-moe}
\icmlauthor{Shuxin Zheng}{ms}
\icmlauthor{Chen Xing}{nankai}
\icmlauthor{Huishuai Zhang}{ms}
\icmlauthor{Yanyan Lan}{cas,ucas}
\icmlauthor{Liwei Wang}{pku-moe,pku-bigdata}
\icmlauthor{Tie-Yan Liu}{ms}
\end{icmlauthorlist}

\icmlaffiliation{cas}{CAS Key Laboratory of Network Data Science and Technology, Institute of Computing Technolog, Chinese Academy of Sciences}
\icmlaffiliation{ucas}{University of Chinese Academy of Sciences}
\icmlaffiliation{pku-moe}{Key Laboratory of Machine Perception, MOE, School of EECS, Peking University}
\icmlaffiliation{pku-bigdata}{Center for Data Science, Peking University, Beijing
Institute of Big Data Research}

\icmlaffiliation{ms}{Microsoft Research}
\icmlaffiliation{nankai}{College of Computer Science, Nankai University}
\icmlcorrespondingauthor{Shuxin Zheng}{shuxin.zheng@microsoft.com}
\icmlcorrespondingauthor{Di He}{dihe@microsoft.com}

\icmlkeywords{Machine Learning, ICML}

\vskip 0.3in
]



\printAffiliationsAndNotice{\textsuperscript{*}Equal contribution \textsuperscript{\dag}Works done while interning at Microsoft Research Asia} % otherwise use the standard text.
% \printAffiliationsAndNotice{\icmlEqualContribution \textsuperscript{\dag}Work done during internship with Microsoft Research Asia} % otherwise use the standard text.

\begin{abstract}
The Transformer is widely used in natural language processing tasks. To train a Transformer however, one usually needs a carefully designed learning rate warm-up stage, which is shown to be crucial to the final performance but will slow down the optimization and bring more hyper-parameter tunings. In this paper, we first study theoretically why the learning rate warm-up stage is essential and show that the location of layer normalization matters. Specifically, we prove with mean field theory that at initialization, for the original-designed Post-LN Transformer, which places the layer normalization between the residual blocks, the expected gradients of the parameters near the output layer are large. Therefore, using a large learning rate on those gradients makes the training unstable. The warm-up stage is practically helpful for avoiding this problem. On the other hand, our theory also shows that if the layer normalization is put inside the residual blocks (recently proposed as Pre-LN Transformer), the gradients are well-behaved at initialization. This motivates us to remove the warm-up stage for the training of Pre-LN Transformers. We show in our experiments that Pre-LN Transformers without the warm-up stage can reach comparable results with baselines while requiring significantly less training time and hyper-parameter tuning on a wide range of applications.
\end{abstract}

\section{Introduction} 
The Transformer \citep{vaswani2017attention} is one of the most commonly used neural network architectures in natural language processing. Layer normalization \citep{lei2016layer} plays a key role in Transformer's success. The originally designed Transformer places the layer normalization between the residual blocks, which is usually referred to as the Transformer with Post-Layer Normalization (Post-LN) \citep{wang2019learning}. This architecture has achieved state-of-the-art performance in many tasks including language modeling \citep{dai2019transformer,al2018character} and machine translation \citep{dehghani2018universal,edunov2018understanding}. Unsupervised pre-trained models based on the Post-LN Transformer architecture also show impressive performance in many downstream tasks \citep{radford2019language,devlin2018bert, yang2019xlnet}. 

Despite its great success, people usually need to deal with the optimization of the Post-LN Transformer more carefully than convolutional networks or other sequence-to-sequence models \citep{popel2018training}. In particular, to train the model from scratch, any gradient-based optimization approach requires a learning rate warm-up stage \citep{vaswani2017attention,liu2019variance}: the optimization starts with an extremely small learning rate, and then gradually increases it to a pre-defined maximum value in a pre-defined number of iterations. Such a warm-up stage not only slows down the optimization process but also brings more hyperparameter tunings. \citet{popel2018training} has shown that the final model performance is quite sensitive to the value of the maximum learning rate and the number of warm-up iterations. Tuning such sensitive hyper-parameters is costly in training large-scale models, e.g., BERT \citep{devlin2018bert} or XLNet \citep{yang2019xlnet}. 

In this paper, we try to alleviate this problem by finding ways to safely remove the learning rate warm-up stage. As the warm-up stage happens in the first several iterations, we investigate the optimization behavior at initialization using mean field theory \citep{lee2017deep,xiao2018dynamical,yang2019mean,yang2019scaling,lee2019wide,zhang2019fixup}. According to our theoretical analysis, when putting the layer normalization between the residual blocks, the expected gradients of the parameters near the output layer are large. Therefore, without the warm-up stage, directly using a large learning rate to those parameters can make the optimization process unstable. Using a warm-up stage and training the model with small learning rates practically avoid this problem. Extensive experiments are provided to support our theoretical findings.

Our theory also shows that the layer normalization plays a crucial role in controlling the gradient scales. This motivates us to investigate whether there are some other ways of positioning the layer normalization that lead to well-behaved gradients. In particular, we study another variant, the Transformer with Pre-Layer Normalization (Pre-LN) \citep{baevski2018adaptive,child2019generating,wang2019learning}. The Pre-LN Transformer puts the layer normalization inside the residual connection and equips with an additional \emph{final-layer normalization} before prediction (Please see Figure \ref{fig:sketch} for the differences between the two variants of the Transformer architectures). We show that at initialization, the gradients are well-behaved without any exploding or vanishing for the Pre-LN Transformer both theoretically and empirically.  

Given the gradients are well-behaved in the Pre-LN Transformer, it is natural to consider removing the learning rate warm-up stage during training. We conduct a variety of experiments, including IWSLT14 German-English translation, WMT14 English-German translation, and BERT pre-training tasks. We show that, in all tasks, the learning rate warm-up stage can be safely removed, and thus, the number of hyper-parameter is reduced. Furthermore, we observe that the loss decays faster for the Pre-LN Transformer model. It can achieve comparable final performances but use much less training time. This is particularly important for training large-scale models on large-scale datasets.
\begin{figure}
% \vspace{-20pt}
  \begin{center}
    \includegraphics[scale=0.35]{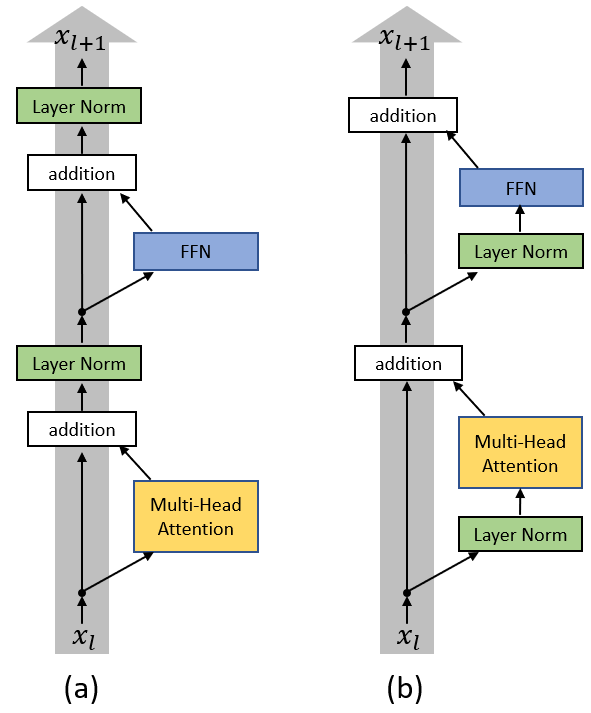}
    \vspace{-5pt}
    \caption{(a) Post-LN Transformer layer; (b) Pre-LN Transformer layer.}
  \end{center}
      \vspace{-25pt}
\end{figure}\label{fig:sketch}

Our contributions are summarized as follows: 

$\bullet$ We investigate two Transformer variants, the Post-LN Transformer and the Pre-LN Transformer, using mean field theory. By studying the gradients at initialization, we provide evidence to show why the learning rate warm-up stage is essential in training the Post-LN Transformer.
    
$\bullet$ We are the first to show that the learning-rate warm-up stage can be removed for the Pre-LN Transformer, which eases the hyperparameter tuning. We further show that by using proper learning rate schedulers, the training time can be largely reduced on a wide range of applications. 

\section{Related work}
Gradient descent-based methods \citep{kingma2014adam, zeiler2012adadelta, duchi2011adaptive,tieleman2012lecture} are popularly used in optimizing deep neural networks. For convolutional neural networks and recurrent neural networks, a relatively large learning rate is usually set in the beginning, and then decreased along with the optimization process \citep{he2016deep,he2017mask, sutskever2014sequence,gehring2017convolutional,he2019bag}. The learning rate warm-up stage has only been shown essential in dealing with some very specific problems, e.g., the large-batch training. \citet{goyal2017accurate,he2019bag,you2018imagenet} showed that a learning rate warm-up stage is preferred when training neural networks with extremely large batch sizes.

However, the learning rate warm-up stage is essential and critical when optimizing the Transformer models in a majority of scenarios \citep{vaswani2017attention, devlin2018bert, dai2019transformer, radford2019language, lu2019understanding}. \citet{popel2018training} investigated the influence of different warm-up strategies for the optimization of the Post-LN Transformer model and found that without or with relatively less warm-up iterations, the optimization diverges. The Pre-LN Transformer has been proposed in several recent works \citep{baevski2018adaptive,child2019generating,wang2019learning} to alleviate some optimization issues when training deeper models, but the troublesome warm-up stage still remains in their training pipelines. 

\citep{liu2019variance} claimed that the benefit of the warm-up stage comes from reducing the variance for the adaptive learning rate in the Adam optimizer \citep{kingma2014adam}. They proposed to rectify the variance of adaptive learning rate by a new variant of Adam called RAdam. However, we find that not only for Adam, the learning rate warm-up stage also helps quite a lot for other optimizers. This may indicate that Adam is not the prerequisite for the necessity of the warm-up stage. In a concurrent and independent work, \citet{nguyen2019transformers} also empirically observed that the Pre-LN Transformer can be trained without learning rate warm-up stage. Our work provides a more comprehensive study regrading this with a theoretical analysis. 

\section{Optimization for the Transformer}
\subsection{Transformer  with Post-Layer Normalization}
The Transformer architecture usually consists of stacked Transformer layers \citep{vaswani2017attention,devlin2018bert}, each of which takes a sequence of vectors as input and outputs a new sequence of vectors with the same shape. A Transformer layer has two sub-layers: the (multi-head) self-attention sub-layer and the position-wise feed-forward network sub-layer. Residual connection \citep{he2016deep} and layer normalization \citep{lei2016layer} are applied for both sub-layers individually. We first introduce each component of the Transformer layer and then present the entire architecture.

\paragraph{Self-attention sub-layer}
An attention function can be formulated as querying an entry with key-value pairs \citep{vaswani2017attention}. The self-attention sub-layer uses scaled dot-product attention, which is defined as: $\text{Attention}(Q,K,V)=\text{softmax}(\frac{QK^T}{\sqrt{d}})V$, where $d$ is the dimensionality of the hidden representations, and $Q$ (Query), $K$ (Key), $V$ (Value) are specified as the hidden representations of the previous layer. The multi-head variant of the self-attention sub-layer is popularly used which allows the model to jointly attend to information from different representation sub-spaces, and is defined as
\begin{align}
\text{Multi-head}(Q,K,V) ={}& \text{Concat} (\text{head}_1,\cdots,\text{head}_H)W^O \nonumber\\
\text{head}_k ={}& \text{Attention}(QW_k^Q, KW_k^K,VW_k^V), \nonumber
\end{align}
where $W_k^Q\in \mathbb{R}^{d \times d_K}, W_k^K\in \mathbb{R}^{d \times d_K}, W_k^V\in \mathbb{R}^{d\times d_V},$ and $W^O\in \mathbb{R}^{Hd_V \times d}$ are project parameter matrices, $H$ is the number of heads. $d_K$ and $d_V$ are the dimensionalities of Key and Value. Without any confusion, given a sequence of vectors $(x_1,...,x_n)$, we use $\text{MultiHeadAtt}(x_{i},[ x_{ 1}, x_{2},\cdots, x_{n}])$ as the multi-head self-attention mechanism on position $i$ which considers the attention from $x_i$ to the entire sequence, i.e., $\text{MultiHeadAtt}(x_{i},[ x_{ 1}, x_{2},\cdots, x_{n}])=\text{Multi-head}(x_i, [x_1, \dots, x_n], [x_1, \dots, x_n])$.

\paragraph{Position-wise FFN sub-layer} 
In addition to the self-attention sub-layer, each Transformer layer contains a fully connected network, which is applied to each position separately and identically. This sub-layer is a two-layer feed-forward network with a ReLU activation function. Given a sequence of vectors $h_1, ..., h_n$, the computation of a position-wise FFN sub-layer on any $h_i$ is defined as: 
\begin{align}
\text{FFN}(h_i)=\text{ReLU}(h_iW^1 + b^1)W^2 + b^2, \nonumber 
\end{align}
where $W^1$, $W^2$, $b^1$ and $b^2$ are parameters.

\paragraph{Residual connection and layer normalization}
Besides the two sub-layers described above, the residual connection and layer normalization are also key components to the Transformer. For any vector $v$, the layer normalization is computed as $\text{LayerNorm}(v) = \gamma \frac{v - \mu}{\sigma} + \beta$, in which $\mu, \sigma$ are the mean and standard deviation of the elements in $v$, i.e., $\mu = \frac{1}{d} \sum_{k=1}^d v_{k}$ and $\sigma^2 = \frac{1}{d}\sum_{k=1}^d(v_{k} -\mu)^2$. Scale $\gamma$ and bias vector $\beta$ are parameters.

Different orders of the sub-layers, residual connection and layer normalization in a Transformer layer lead to variants of Transformer architectures. One of the original and most popularly used architecture for the Transformer and BERT \citep{vaswani2017attention, devlin2018bert} follows ``self-attention (FFN) sub-layer $\rightarrow$ residual connection $\rightarrow$ layer normalization'', which we call the Transformer with Post-Layer normalization (Post-LN Transformer), as illustrated in Figure \ref{fig:sketch}.

\paragraph{Post-LN Transformer} Denote $x_{l, i}$ as the input of the $l$-th Transformer layer at position $i$, where $x_{l,i}$ is a real-valued vector of dimension $d$, $i=1,2,...,n$, $l=1,2,...,L$. $n$ is the length of the sequence and $L$ is the number of layers. For completeness, we define $x_{0,i}$ as the input embedding at position $i$ which is usually a combination of word embedding and positional embedding. The computations inside the $l$-th layer are composed of several steps, and we use super-scripts on $x$ to present the input(output) of different steps as in Table 1 (left), where $W^{1,l}$, $W^{2,l}$, $b^{1,l}$ and $b^{2,l}$ are parameters of the FFN sub-layer in the $l$-th layer.

\begin{table*}[tb]
\centering
\caption{Post-LN Transformer v.s. Pre-LN Transformer}
\vspace{2pt} 
\label{tbl:translation}
\centerline{
\begin{tabular}{ll}
\toprule
 Post-LN Transformer & Pre-LN Transformer\\
\midrule
 $x^{post,1}_{l,i} = \text{MultiHeadAtt}(x^{post}_{l, i},[ x^{post}_{l, 1}, \cdots, x^{post}_{l, n}])$ & $x^{pre,1}_{l,i}=\text{LayerNorm}(x^{pre}_{l,i})$  \\
 $x^{post,2}_{l,i}=x^{post}_{l,i}+x^{post,1}_{l,i}$ &$x^{pre,2}_{l,i}=\text{MultiHeadAtt}(x_{l, i}^{pre,1},[ x_{l, 1}^{pre,1}, \cdots, x_{l, n}^{pre,1}])$  \\
 $x^{post,3}_{l,i}=\text{LayerNorm}(x^{post,2}_{l,i})$ & $x^{pre,3}_{l,i}=x^{pre}_{l,i}+x^{pre,2}_{l,i}$ \\
 $x^{post,4}_{l,i}=\text{ReLU}(x^{post,3}_{l,i}W^{1,l} + b^{1,l})W^{2,l} + b^{2,l}$ & $x^{pre,4}_{l,i}=\text{LayerNorm}(x^{pre,3}_{l,i})$ \\
$x^{post,5}_{l,i}=x^{post,3}_{l,i} + x^{post,4}_{l,i}$ &  $x^{pre,5}_{l,i}=\text{ReLU}(x^{pre,4}_{l,i}W^{1,l} + b^{1,l})W^{2,l} + b^{2,l}$  \\
  $x^{post}_{l+1,i}=\text{LayerNorm}(x^{post,5}_{l,i})$ &$x^{pre}_{l+1,i}=x^{pre,5}_{l,i}+x^{pre,3}_{l,i}$ \\
\midrule
 & Final LayerNorm: $x_{Final,i}^{pre}\leftarrow\text{LayerNorm}(x_{L+1,i}^{pre})$ \\
\bottomrule
\end{tabular}   
}
\end{table*}
\subsection{The learning rate warm-up stage}
We are interested in the \emph{learning rate warm-up} stage in the optimization of the Post-LN Transformer. Different from the optimization of many other architectures in which the learning rate starts from a relatively large value and then decays \citep{bahdanau2017neural,dauphin2017language}, a learning rate warm-up stage for the Post-LN Transformer seems \textbf{critical} \citep{popel2018training}. We denote the learning rate of the $t$-th iteration as $\text{lr}(t)$ and the maximum learning rate during training as $\text{lr}_{max}$. Given a predefined time frame $T_{\text{warmup}}$, the learning rate scheduler for the first $T_{\text{warmup}}$ iterations \citep{tensor2tensor} is defined as 
\begin{align}\label{eqn:warm-up-lr}
\text{lr}(t)=\frac{t}{T_{\text{warmup}}}\text{lr}_{max}, t\leq T_{\text{warmup}}. 
\end{align}
After this warm-up stage, the learning rate will be set by classical learning rate schedulers, such as the linear decay, the inverse square-root decay, or forced decay at particular iterations. We conduct experiments to show that this learning rate warm-up stage is essential for training Post-LN Transformer models.
\begin{figure*}[htb]
\centering
\subfigure[Loss/BLEU on the IWSLT14 De-En task (Adam)]{
\label{fig:v1-adam}
\begin{minipage}[t]{0.49\linewidth}
\centering
\includegraphics[width=\linewidth]{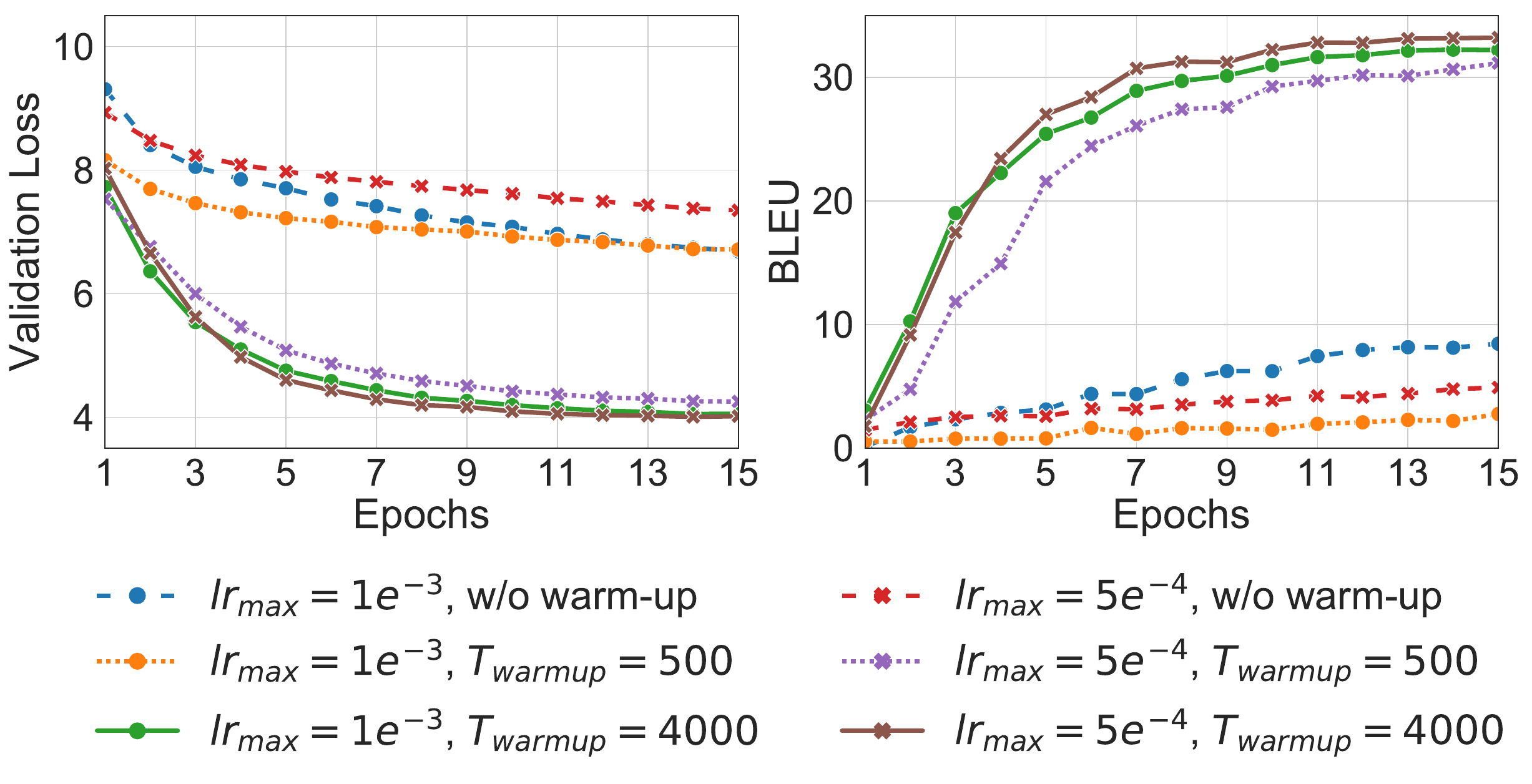}
\end{minipage}%
}
\subfigure[Loss/BLEU on the IWSLT14 De-En task (SGD)]{
\label{fig:v1-sgd}
\begin{minipage}[t]{0.49\linewidth}
\centering
\includegraphics[width=\linewidth]{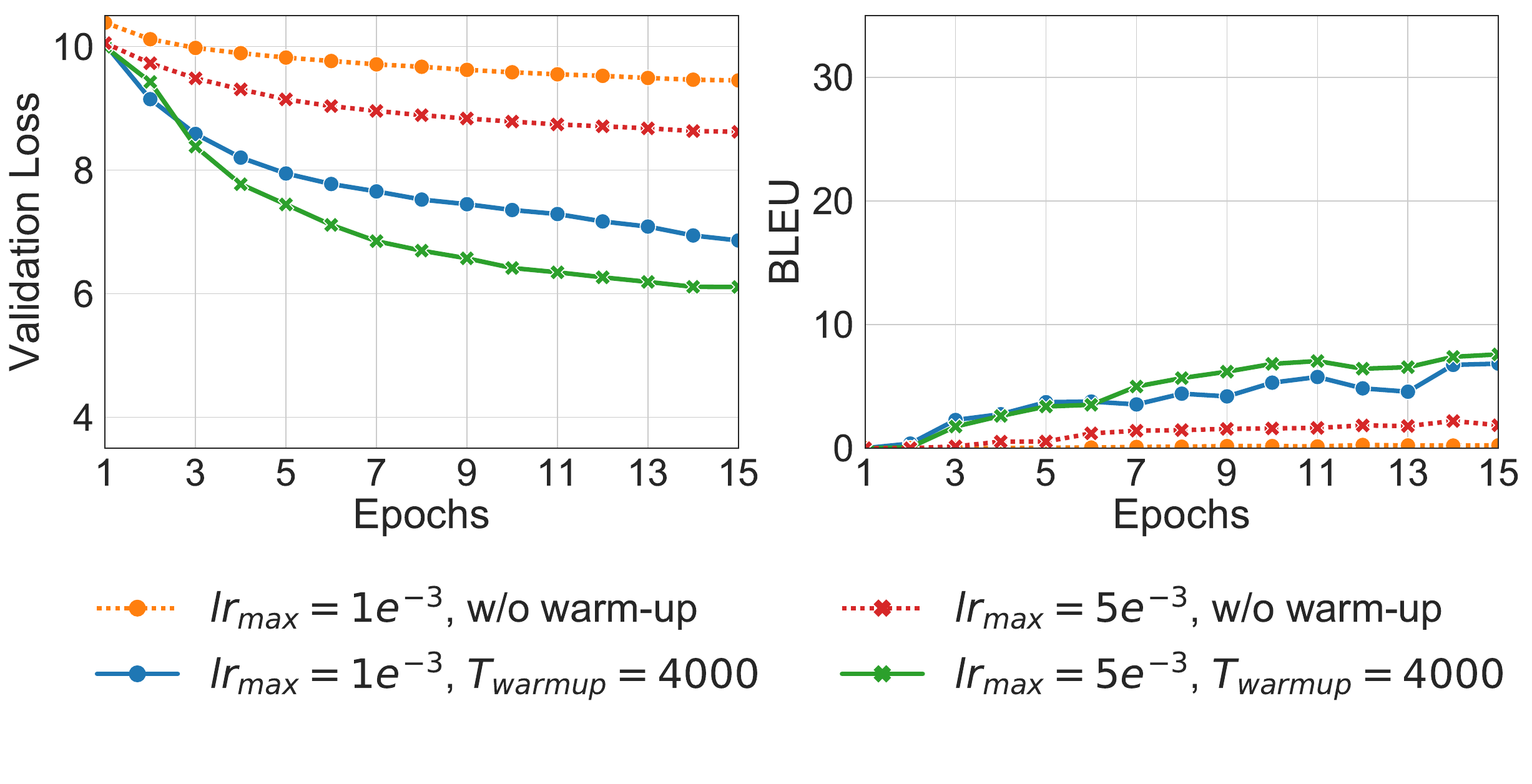}
\end{minipage}%
}
\centering
\caption{Performances of the models optimized by Adam and SGD on the IWSLT14 De-En task.}
\end{figure*}
\paragraph{Experimental setting} %Machine translation is an important application for natural language processing and the Transformer architecture is the most commonly used model architecture in translation. 
We conduct experiments on the IWSLT14 German-to-English (De-En) machine translation task. We mainly investigate two aspects: whether the learning rate warm-up stage is essential and whether the final model performance is sensitive to the value of $T_{\text{warmup}}$. To study the first aspect, we train the model with the Adam optimizer \citep{kingma2014adam} and the vanilla SGD optimizer \citep{ruder2016overview} respectively. For both optimziers, we check whether the warm-up stage can be removed. We follow \citet{vaswani2017attention} to set hyper-parameter $\beta$ to be $(0.9,0.98)$ in Adam. We also test different $\text{lr}_{max}$ for both optimizers. For Adam, we set $\text{lr}_{max}=5e^{-4}$ or $1e^{-3}$, and for SGD, we set $\text{lr}_{max}=5e^{-3}$ or $1e^{-3}$. When the warm-up stage is used, we set $T_{\text{warmup}}=4000$ as suggested by the original paper \citep{vaswani2017attention}. To study the second aspect, we set $T_{\text{warmup}}$ to be 1/500/4000 (``1'' refers to the no warm-up setting) and use $\text{lr}_{max}=5e^{-4}$ or $1e^{-3}$ with Adam. For all experiments, a same inverse square root learning rate scheduler is used after the warm-up stage. We use both validation loss and BLEU \citep{papineni2002bleu} as the evaluation measure of the model performance.

\paragraph{Results and discussions} We record the model checkpoints for every epoch during training and calculate the validation loss and BLEU score. The performance of the models are plotted in Figure \ref{fig:v1-adam} and Figure \ref{fig:v1-sgd}. The x-axis is the epoch number and the y-axis is the BLEU score/validation loss. "w/o warm-up" indicates ``without the warm-up stage'' while "w/ warm-up" indicates ``with the warm-up stage''.

First, we can see that for both optimizers, the learning rate warm-up stage is essential. Without the warm-up stage, the BLEU score of the model trained with Adam optimizer can only achieve 8.45. As a comparison, the model trained using the warm-up stage can achieve around 34 in terms of BLEU score. The same trend can also be observed on the validation loss curves. Although the performance of the model trained with SGD is significantly worse than Adam, we can still see similar phenomena as Adam. The BLEU score is just above zero in 15 epochs without using the warm-up stage. 

Second, we can see that the optimization process is sensitive to the value of $T_{\text{warmup}}$, which means $T_{\text{warmup}}$ is an important hyper-parameter in training the Post-LN Transformer. For example, when setting $T_{\text{warmup}}=500$, the learned models with Adam achieve only 31.16 and 2.77 in term of BLEU score for $lr_{max}=5e^{-4}$ and $1e^{-3}$ respectively. 

Such a warm-up stage has several disadvantages. First, its configuration significantly affects the final performance. The practitioners need a careful hyper-parameter tuning, which is computationally expensive for large-scale NLP tasks. Second, the warm-up stage could slow down the optimization. Standard optimization algorithms usually start with a large learning rate for fast convergence. However, when using the warm-up stage, the learning rate has to gradually increase from zero, which may make the training inefficient. 
\citet{liu2019variance} suggests that the warm-up stage plays a role in reducing the undesirably significant variance in Adam in the early stage of model training. However, according to our results, the warm-up stage also helps the training of SGD. This suggests that the benefit of the warm-up stage may be not for a particular optimizer.

\subsection{Understanding the Transformer at initialization}

We can see that the Post-LN Transformer cannot be trained with a large learning rate from scratch. This motivates us to investigate what happens at the model initialization. We first introduce the parameter initialization setting for our theoretical analysis and then present our theoretical findings.

\paragraph{Notations}
We denote $\mathcal{L}(\cdot)$ as the loss function of one position, $\tilde{\mathcal{L}}(\cdot)$ as the loss function of the whole sequence, $\|\cdot\|_2$ and $\|\cdot\|_F$ as the $l_2$ norm (spectral norm) and the Frobenius norm, $\text{LN}(x)$ as the standard layer normalization with scale $\gamma=1$ and bias $\beta=0$, and $\mathbf{J}_{LN}(x)=\frac{\partial \text{LN}(x)}{\partial x}$ as the Jacobian matrix of $\text{LN}(x)$. Let $\mathcal{O}(\cdot)$ denote standard Big-O notation that suppress multiplicative
constants.

\paragraph{Parameter Initialization}
The parameter matrices in each Transformer layer are usually initialized by the Xavier initialization \citep{glorot2010understanding}. Given a matrix of size $n_{in}\times n_{out}$, the Xavier initialization sets the value of each element by independently sampling from Gaussian distribution $N(0, \frac{2}{n_{in}+n_{out}})$. The bias vectors are usually initialized as zero vectors. The scale $\gamma$ in the layer normalization is set to one. 

For theoretical analysis, we study a simpler setting. First, we focus on single-head attention instead of the multi-head variant and for all layers, we set the shape of $W^{Q,l}$, $W^{K,l}$, $W^{V,l}$, $W^{1,l}$,$W^{2,l}$ to be $d\times d$. Second, we initialize the parameter matrices in the self-attention sub-layer $W^{Q,l}$ and $W^{K,l}$ to be zero matrices. In this setting, the attention is a uniform distribution at initialization and $\text{MultiHeadAtt}(x_{l, i}^1,[ x_{l, 1}^1, x_{l, 2}^1,\cdots, x_{l, n}^1])$ can be simplified as $\frac{1}{n}\sum_{j=1}^{n}x_{l, j} W^{V, l}$. Third, we assume the input vectors are also sampled from the same Gaussian distribution. This is reasonable since the inputs are linear combinations of word embeddings and learnable positional embeddings, both of which are initialized by Gaussian distributions. 

\paragraph{Post-LN Transformer v.s. Pre-LN Transformer}
We compare the Post-LN Transformer with another variant of the Transformer architecture, the Transformer with Pre-Layer Normalization (Pre-LN). The Pre-LN Transformer was implemented in several systems \citep{tensor2tensor, klein2018opennmt,liu2019roberta}. \citet{wang2019learning} suggested that the Pre-LN Transformer outperforms the Post-LN Transformer when the number of layers increases. Different from the Post-LN Transformer that puts the layer normalization between the residual blocks, the Pre-LN Transformer puts the layer normalization inside the residual connection and places it before all other non-linear transformations. Additionally, the Pre-LN Transformer uses a \emph{final layer normalization} right before the prediction. We provide the mathematical formulations and visualizations of the Post-LN/Pre-LN Transformer in Table \ref{tbl:translation} and Figure \ref{fig:sketch}. 

For both architectures, each $x_{L,i}$  passes through a softmax layer to produce a distribution over the dictionary $V$. The loss function is defined on the softmax distribution. For example, in sequence prediction, the loss function is defined as $\mathcal{L}(x_{L+1,i}^{post})=-\log(\text{softmax}_{y_i}(W^{emb}x_{L+1,i}^{post}))$ for the Post-LN Transformer and $\mathcal{L}(x_{Final,i}^{pre})=-\log(\text{softmax}_{y_i}(W^{emb}x_{Final,i}^{pre}))$ for the Pre-LN Transformer, where $\text{softmax}_{y_i}$ is the probability of ground truth token $y_i$ outputted by the softmax distribution and $W^{emb}$ is the word embedding matrix. The loss of the whole sequence is an average of the loss on each position. Without loss of generality, we assume that all the derivatives are bounded. We introduce the following concentration property of random variables which will be further used in the theorem.

\begin{definition}
A random variable $Z\geq 0$ is called $(\epsilon,\delta)$-bounded if with probability at least $1-\delta$, $\frac{Z-\mathbb{E}Z}{\mathbb{E}Z}\leq \epsilon$, where $\epsilon>0$ and $0<\delta<1$. %where $\delta$ is a neglectable constant compared with $\epsilon$.
\end{definition}

Intuitively, if the random variable $Z$ is $(\epsilon,\delta)$-bounded, then with a high probability its realization will not get too far away from its expectation. For example, if $Y$ is a $d$-dimensional standard Gaussian random vector, then $Z=\|Y\|_2^2$ is $(\epsilon,\delta)$-bounded with $\delta=\exp(-d\epsilon^2/8)$, $0<\epsilon<1$ (see 
 supplementary material for details). As parameter matrices in self-attention sub-layers and FFN sub-layers are initialized by Gaussian distributions, if the norm of the hidden states in the Transformer satisfies the concentrated condition above, we have the following theorem to characterize the scale of the gradients.

\begin{theorem} [Gradients of the last layer in the Transformer]
Assume that $\|x^{post,5}_{L,i}\|_2^2$ and $\|x^{pre}_{L+1,i}\|_2^2$ are $(\epsilon,\delta)$-bounded for all $i$, where $\epsilon$ and $\delta=\delta(\epsilon)$ are small numbers.  Then with probability at least $0.99-\delta-\frac{\epsilon}{0.9+\epsilon}$, for the Post-LN Transformer with $L$ layers, the gradient of the parameters of the last layer satisfies $$\|\frac{\partial \tilde{\mathcal{L}}}{\partial W^{2,L}}\|_F\leq \mathcal{O}(d\sqrt{\ln{d}})$$ and for the Pre-LN Transformer with $L$ layers, $$\|\frac{\partial \tilde{\mathcal{L}}}{\partial W^{2,L}}\|_F\leq \mathcal{O}\left(d\sqrt{\frac{\ln{d}}{L}}\right).$$
\end{theorem}

From Theorem 1, we can see that for the Post-LN Transformer, the scale of the gradients to the last FFN layer is of order $\mathcal{O}(d\sqrt{\ln{d}})$ which is independent of $L$. For the Pre-LN Transformer, the scale of the gradients is much smaller. 
% Because all the parameters are gaussian initialized and the dimension $d$ is usually large, we expect $\|x^{post,5}_{L,k}\|_2^2$ and $\|x^{pre}_{L+1,k}\|_2^2$ to behave like $\chi^2$ random variables, so $\delta(\epsilon)=\mathcal{O}(\sqrt{\epsilon}\operatorname{exp}(-\frac{\epsilon}{2}))$.
We first study the forward propagation of the Post-LN Transformer and the Pre-LN Transformer. Lemma 1 will be served as a basic tool to prove the main theorem and other lemmas.

\begin{figure*}[htbp]
\label{fig:decoder-expectation}
\centering
\subfigure[$W^{1}$ in the FFN sub-layers]{
\label{fig:decoder-expectation1}
\begin{minipage}[t]{0.24\linewidth}
\centering
\includegraphics[width=\linewidth]{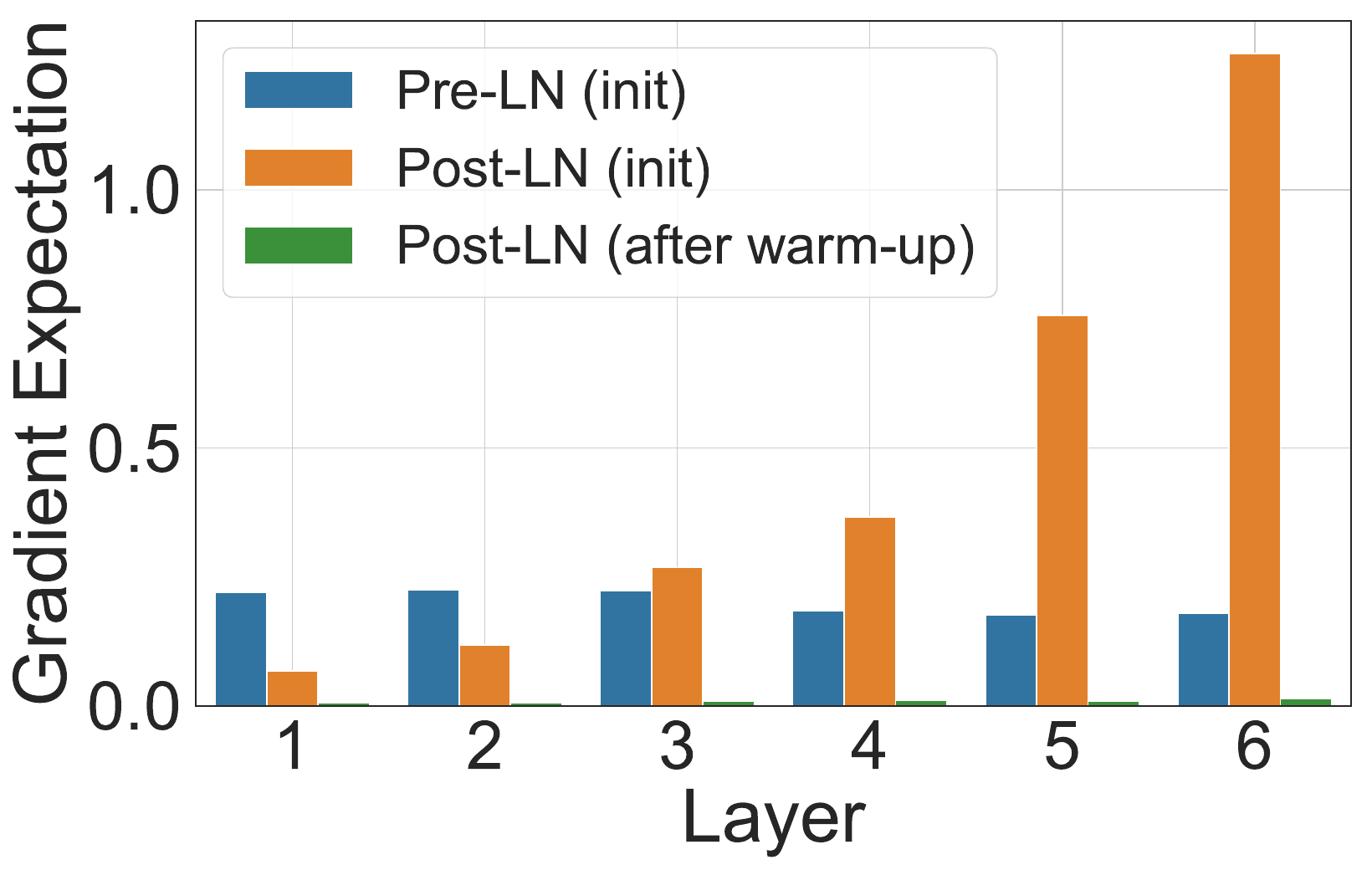}
% \caption{fig1}
\end{minipage}%
% \caption{fig1}
}%
\subfigure[$W^{2}$ in the FFN sub-layers]{
\label{fig:decoder-expectation2}
\begin{minipage}[t]{0.24\linewidth}
\centering
\includegraphics[width=\linewidth]{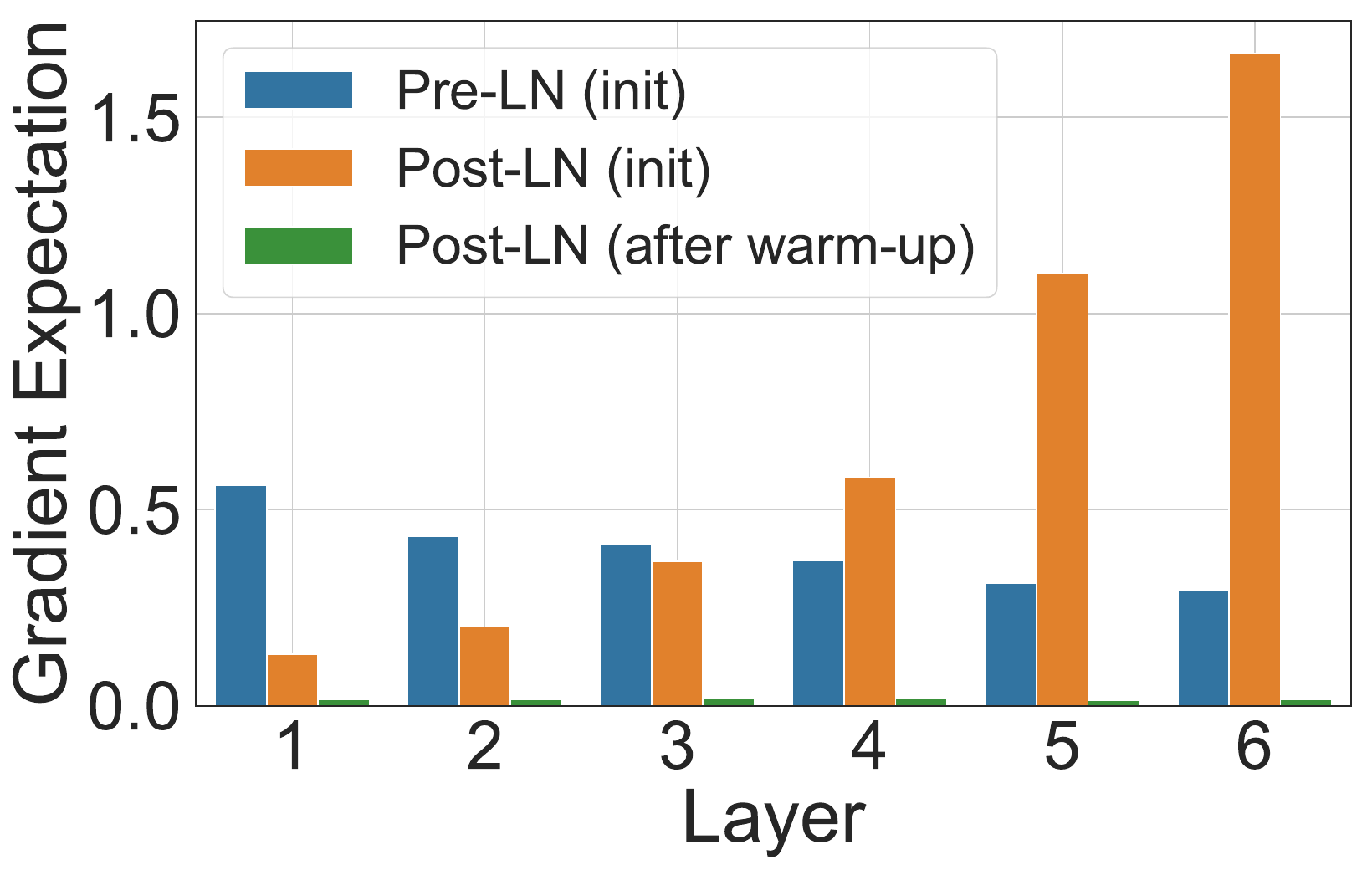}
%\caption{fig2}
\end{minipage}%
}%
\centering
%\caption{Norm of expected gradients for Pre-LN/Post-LN Transformer}
%\end{figure*}
%\begin{figure*}[htbp]
%\label{fig:decoder-expectation-different-L}
%\centering
\subfigure[Pre-LN Transformer]{
\label{fig:decoder-expectation-different-L-preln}
\begin{minipage}[t]{0.24\linewidth}
\centering
\includegraphics[width=\linewidth]{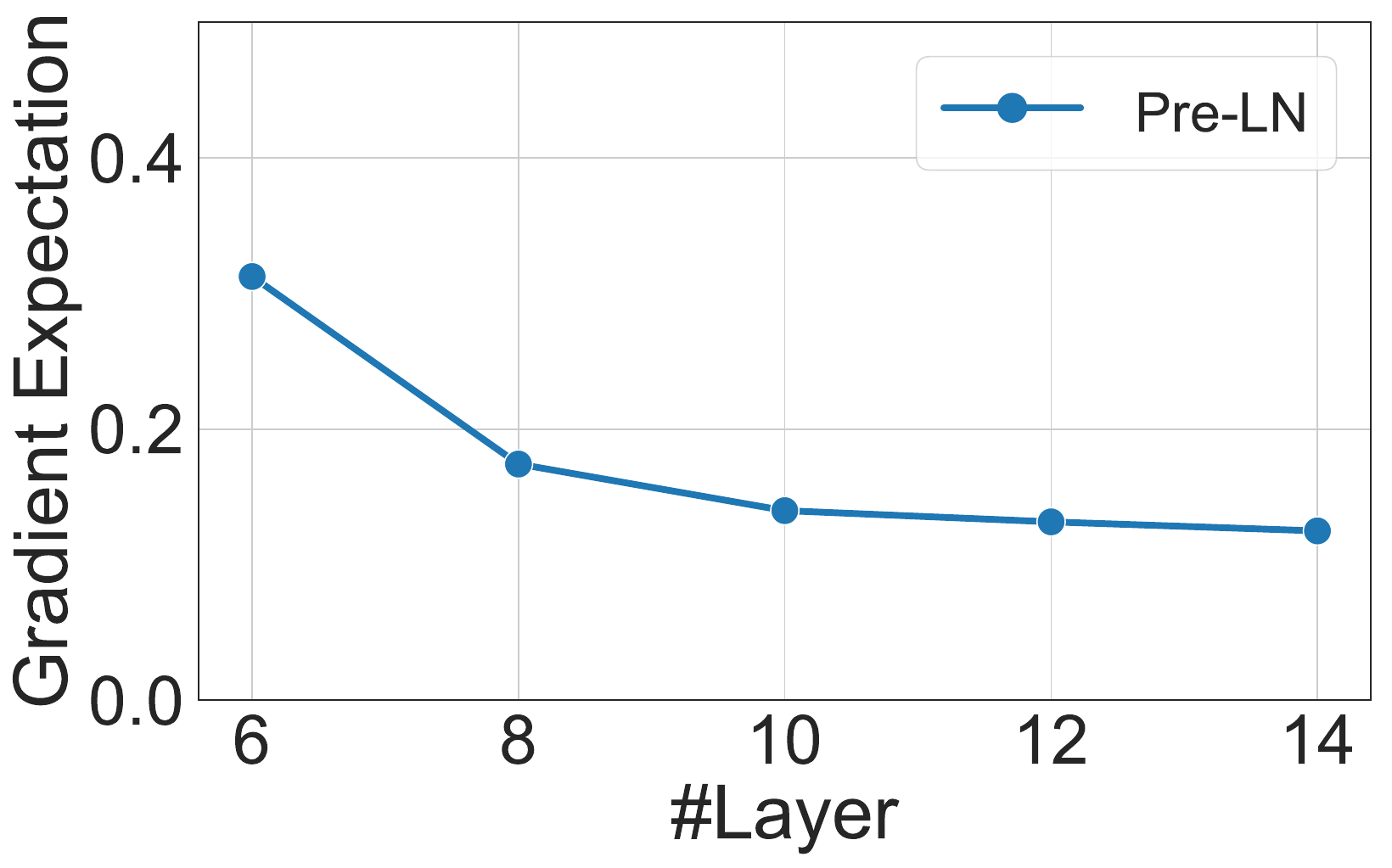}
%\caption{fig1}
\end{minipage}%
}%
\subfigure[Post-LN Transformer]{
\label{fig:decoder-expectation-different-L-postln}
\begin{minipage}[t]{0.24\linewidth}
\centering
\includegraphics[width=\linewidth]{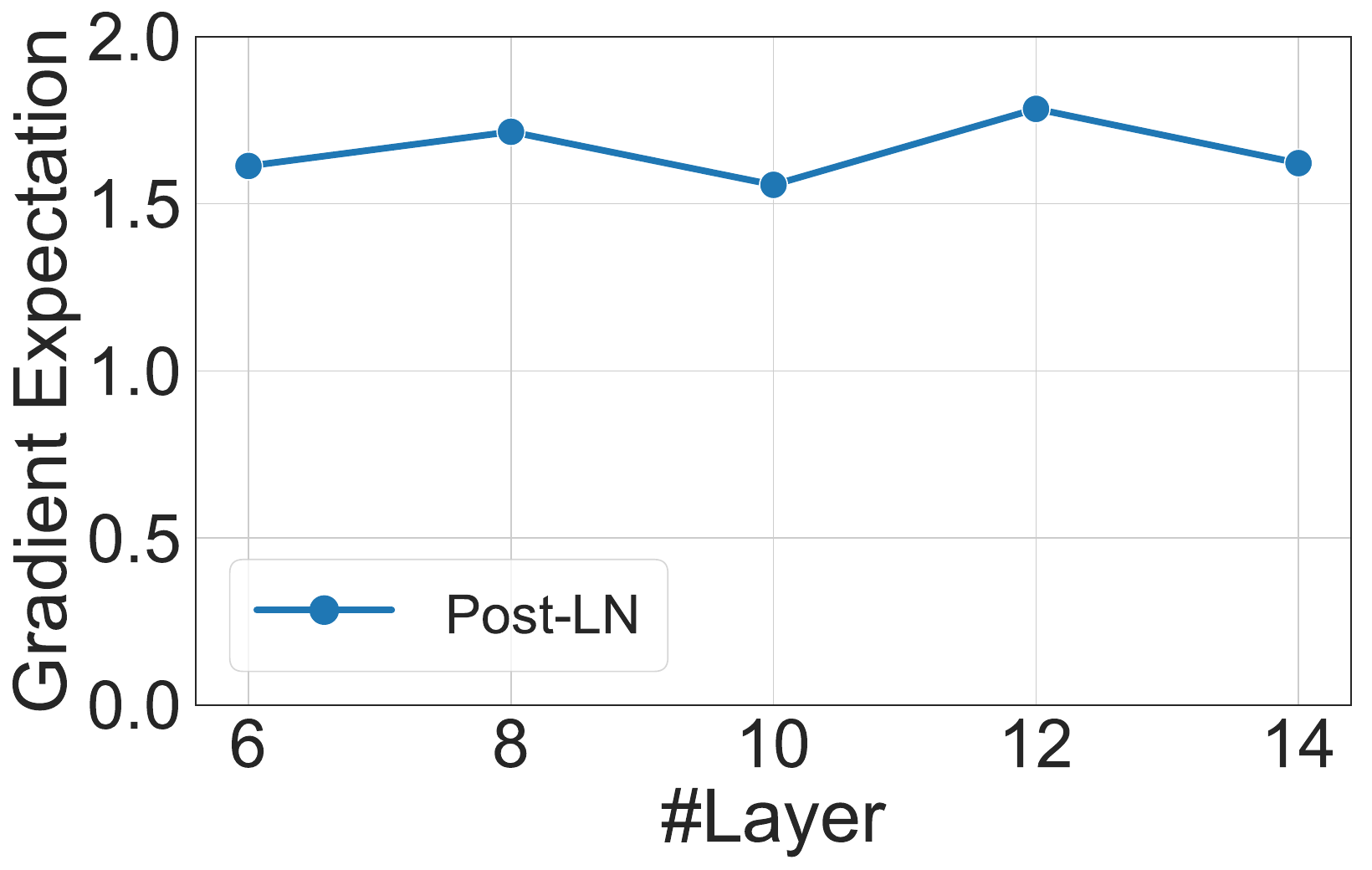}
%\caption{fig2}
\end{minipage}%
}%
\centering
\caption{The norm of gradients of 1. different layers in the 6-6 Transformer (a,b). 2. $W^{2,L}$ in different size of the Transformer (c,d).}
\end{figure*}

\begin{lemma}\label{l1}
If $X\in \mathbb{R}^d$ is a Gaussian vector, $X\sim N(0,\sigma^2 \mathbf{I}_d)$, then $\mathbb{E}(\|\text{ReLU}(X)\|_2^2)=\frac{1}{2}\sigma^2 d$.
\end{lemma}

Based on Lemma 1, we have the following lemma to estimate the scale of the hidden states in different layers for the Post-LN Transformer and the Pre-LN Transformer.

\begin{lemma}
At initialization, for the Post-LN Transformer, $\mathbb{E}(\|x_{l,i}^{post,5}\|_2^2)=\frac{3}{2}d$ for all $l>0$ and $i$. For the Pre-LN Transformer, $(1+\frac{l}{2})d\leq\mathbb{E}(\|x^{pre}_{l,i}\|_2^2)\leq(1+\frac{3l}{2})d$ for all $l>0$ and $i$. Expectations are taken over the input and the randomness of initialization.
\end{lemma}

Lemma 2 studies the expected norm of the hidden states in both Post-LN/Pre-LN Transformer. It is obviously that in the Post-LN Transformer, the norm of $x_{l,i}^{post}$ is $\sqrt{d}$ and thus we study the norm of $x^{post,5}_{l,i}$ instead. As we can see from Lemma 2, the scale of the hidden states in the Post-LN Transformer keeps to be the same in expectation while the scale of the hidden states in the Pre-LN Transformer grows linearly along with the depth. The next lemma shows that the scale of the hidden states highly relates to the scale of the gradient in the architectures using layer normalization.

\begin{lemma}\label{l2}
For $x\in \mathbb{R}^d$, we have $\|\mathbf{J}_{LN}(x)\|_2=\mathcal{O}(\frac{\sqrt{d}}{\|x\|_2})$
in which $\mathbf{J}_{LN}(x)=\frac{\partial\text{LN}(x)}{\partial x}$.
\end{lemma}

The proof of Lemma 1, Lemma 2, Lemma 3, and Theorem 1 can be found in the supplementary material. The main idea is that the layer normalization will normalize the gradients. In the Post-LN Transformer, the scale of the inputs to the layer normalization is independent of $L$, and thus the gradients of parameters in the last layer are independent of $L$. While in the Pre-LN Transformer, the scale of the input to the final layer normalization is linear in $L$, and thus the gradients of all parameters will be normalized by $\sqrt{L}$.

\paragraph{Extended theory to other layers/parameters}
We have provided a formal proof on the gradients of the last FFN sub-layer as above. In order to fully understand the optimization, we also make some preliminary analysis for other layers and other parameters. Our main result is that the gradient norm in the Post-LN Transformer is large for the parameters near the output and will be likely to decay as the layer index $l$ decreases. On the contrary, the gradient norm in the Pre- Transformer will be likely to stay the same for any layer $l$. All the preliminary theoretical results are provided in the supplementary material.

\subsection{Empirical verification of the theory and discussion}
As our theory is derived based on several simplifications of the problem, we conduct experiments to study whether our theoretical insights are consistent with what we observe in real scenarios. The general model and training configuration exactly follow Section 3.2. The experiments are repeated ten times using different random seeds. 

\paragraph{On the concentration property} Given an initialized model, we record the hidden states in the Post-LN/Pre-LN Transformer across batches and find that the norm of the hidden states satisfies the property ((0.1,0.125)-bounded).

\paragraph{On Theorem 1} Theorem 1 suggests that for any sizes of the Post-LN Transformer, the scale of the gradient norm in the last FFN sub-layer remains the same. On the contrary, that of the Pre-LN Transformer decreases as the size of the model grows. We calculate and record the gradient norm in the last FFN sub-layer in 6-6/8-8/10-10/12-12/14-14 Post-LN/Pre-LN Transformer models at initialization. The results are plotted in Figure \ref{fig:decoder-expectation-different-L-preln} and \ref{fig:decoder-expectation-different-L-postln}. The x-axis is the size of the model, and the y-axis is the value of the gradient norm of $W^2$ in the final FFN sub-layer. The figures show when the number of layers grows, the gradient norm remains in the Post-LN Transformer (around 1.6) and decreases in the Pre-LN Transformer. This observation is consistent with our theory.

\paragraph{On the extended theory} %
We calculate the gradient norm of each paramter matrix in 6-6 Post-LN/Pre-LN Transformer. We record the gradient for each parameter for different mini-batches. For elements in a parameter matrix, we calculate their expected gradients and use the Frobenius norm of those values as the scale of the expected gradient of the matrix.  Figure \ref{fig:decoder-expectation1} and \ref{fig:decoder-expectation2} shows those statistics for FFN sub-layers. The x-axis indexes different Transformer layers. It can be seen from the figure, the scale of the expected gradients grows along with the layer index for the Post-LN Transformer. On the contrary, the scale almost keeps the same for different layers in the Pre-LN Transformer. These observations are consistent with our theoretical findings.

\paragraph{The critical warm-up stage for Post-LN Transformer
} Given the analysis above, we hypothesize that the gradient scale is one of the reasons that the Post-LN Transformer needs a careful learning rate scheduling. Since the gradients are large for some layers, using a large learning rate without warm-up may make the training unstable. 

To verify this argument, first, we study the gradient statistics for the Post-LN Transformer after the warm-up stage with Adam. It can be seen from Figure \ref{fig:decoder-expectation1} and \ref{fig:decoder-expectation2} that the scale of the gradients are very small, and the model can be trained with large learning rates. Second, we conduct an experiment to train the Post-LN Transformer from scratch using a fixed small learning rate, i.e., $1e^{-4}$,  to verify whether using small-step updates mitigates the issue. The details are provided in the supplementary material. In general, using a very small and fixed learning rate can mitigate the problem and optimize the Post-LN Transformer to a certain extent but the convergence is significantly slower. Both experiments above are supportive to our claim.

\section{Experiments}
We find in the previous section that the gradients at initialization for Pre-LN Transformer are well-behaved. Given this observation, we deduce that the learning rate warm-up stage can be safely removed when training Pre-LN Transformer. In this section, we empirically verify it on two main tasks in NLP, machine translation and unsupervised pre-training.

\subsection{Experiment Settings}
\paragraph{Machine Translation}
We conduct our experiments on two widely used tasks: the IWSLT14 German-to-English (De-En) task and the WMT14 English-to-German (En-De) task. 
For the IWSLT14 De-En task, we use the same model configuration as in Section 3. For the WMT14 En-De task, we use the Transformer \texttt{base} setting. More details can be found in the supplementary material.  

\begin{figure*}[htb]
\label{fig:v1v2-compare-machine-translation}
\centering
\subfigure[Validation Loss (IWSLT)]{
\label{fig:v1v2-iwslt-loss}
\begin{minipage}[t]{0.24\linewidth}
\centering
\includegraphics[width=\linewidth]{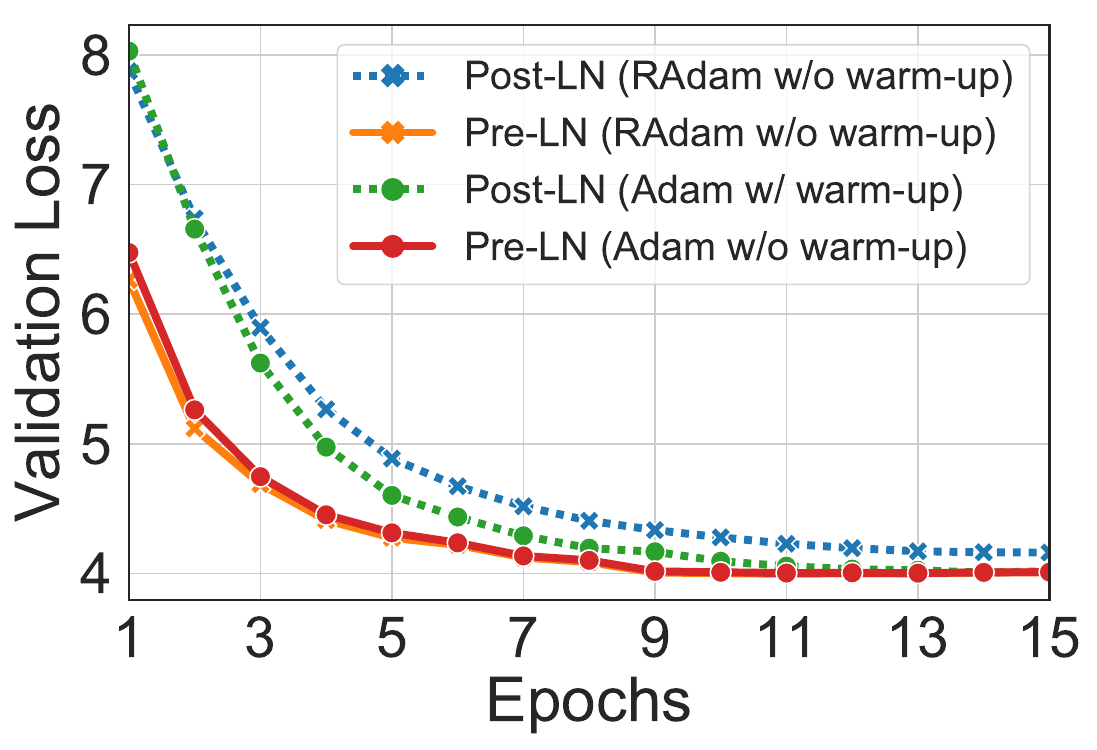}
\end{minipage}%
}%
\subfigure[BLEU (IWSLT)]{
\label{fig:v1v2-iwslt-bleu}
\begin{minipage}[t]{0.24\linewidth}
\centering
\includegraphics[width=\linewidth]{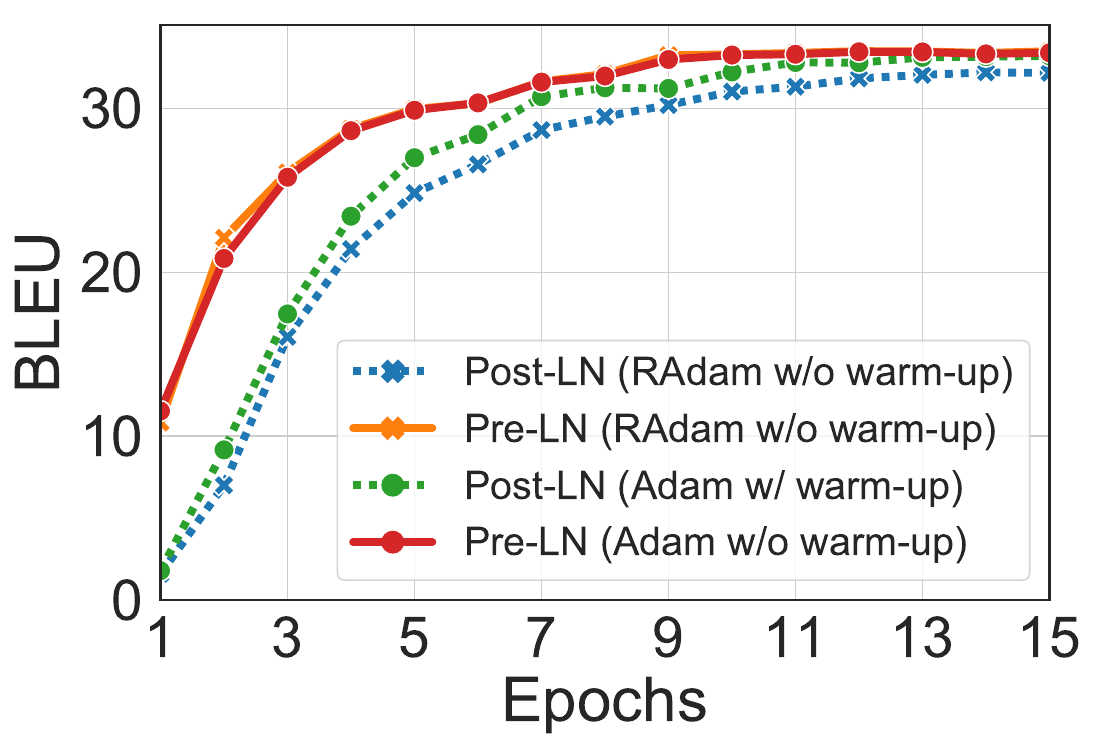}
\end{minipage}%
}
\subfigure[Validation Loss (WMT)]{
\label{fig:v1v2-wmt-loss}
\begin{minipage}[t]{0.24\linewidth}
\centering
\includegraphics[width=\linewidth]{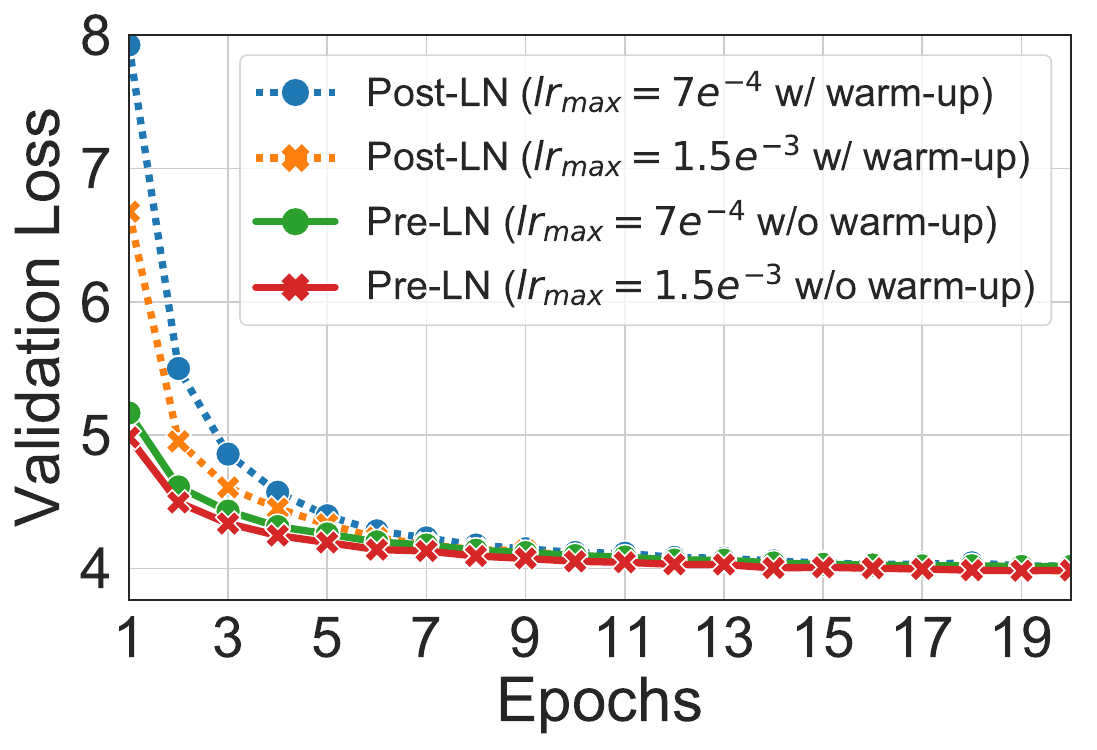}
\end{minipage}
}%
\subfigure[BLEU (WMT)]{
\label{fig:v1v2-wmt-bleu}
\begin{minipage}[t]{0.24\linewidth}
\centering
\includegraphics[width=\linewidth]{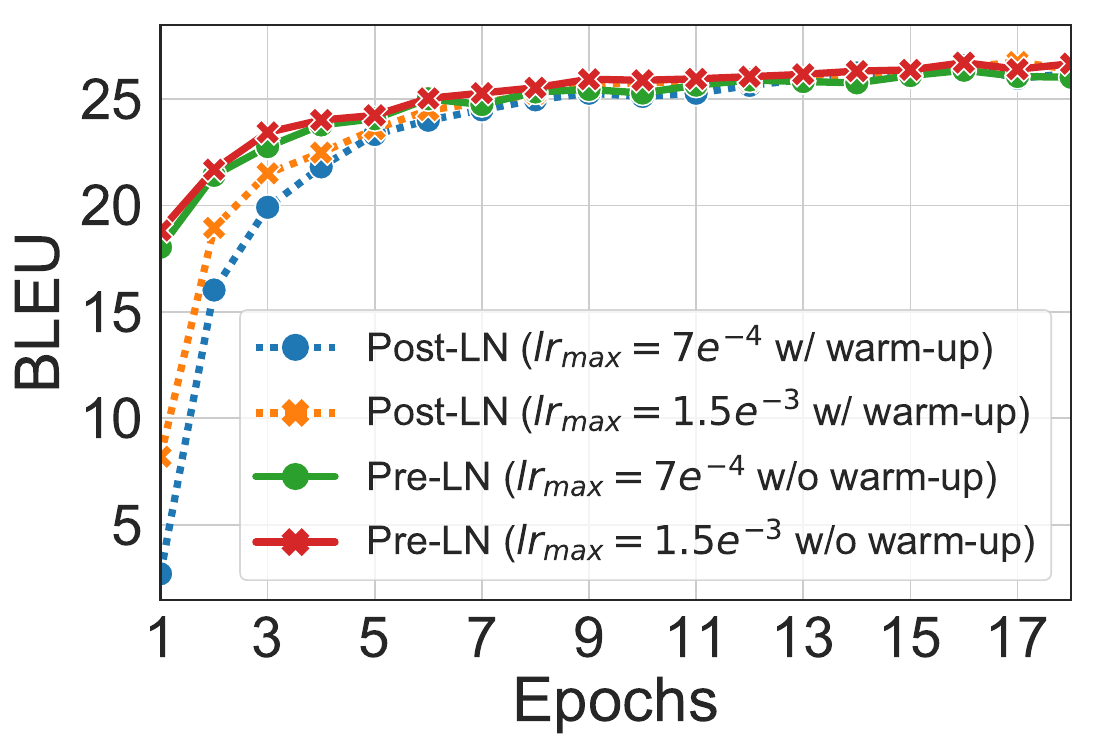}
\end{minipage}
}%
\centering
\caption{Performances of the models on the IWSLT14 De-En task and WMT14 En-De task}
\end{figure*}

For training the Pre-LN Transformer, we remove the learning rate warm-up stage. On the IWSLT14 De-En task, we set the initial learning rate to be $5e^{-4}$ and decay the learning rate at the 8-th epoch by 0.1. On the WMT14 En-De task, we run two experiments in which the initial learning rates are set to be $7e^{-4}/1.5e^{-3}$ respectively. Both learning rates are decayed at the 6-th epoch followed by the inverse square root learning rate scheduler.

We train the Post-LN Transformer using the learning rate warm-up stage as the baseline. In both IWSLT14 De-En task and WMT14 En-De task, we set the number of the warm-up stage to be 4000 following \citet{vaswani2017attention} and then use the inverse square root learning rate scheduler. For all experiments above, we use the Adam optimizer and set the hyper-parameter $\beta$ to be $(0.9,0.98)$. We set $lr_{max}$ as same as the initial learning rates of the Pre-LN Transformer in each corresponding experiment. Since \citet{liu2019variance} suggests that the learning rate warm-up stage can be removed using RAdam, we try this optimizer on the IWSLT14 De-En task. We use linear learning rate decay suggested by \citet{liu2019variance} and keep all other hyper-parameters to be the same as in other experiments.

\paragraph{Unsupervised Pre-training (BERT)}
We follow \citep{devlin2018bert} to use English Wikipedia corpus and BookCorpus for pre-training. As the dataset BookCorpus \citep{moviebook} is no longer freely distributed. We follow the suggestions from \citep{devlin2018bert} to crawl and collect BookCorpus on our own. The concatenation of two datasets contains roughly 3.4B words in total, which is comparable with the data corpus used in \citep{devlin2018bert}. We randomly split documents into one training set and one validation set. The training-validation ratio for pre-training is 199:1.

\begin{figure*}[htb]
\centering
\subfigure[Validation Loss on BERT]{
\label{fig:v1v2-bert}
\begin{minipage}[t]{0.25\linewidth}
\centering
\includegraphics[width=\linewidth]{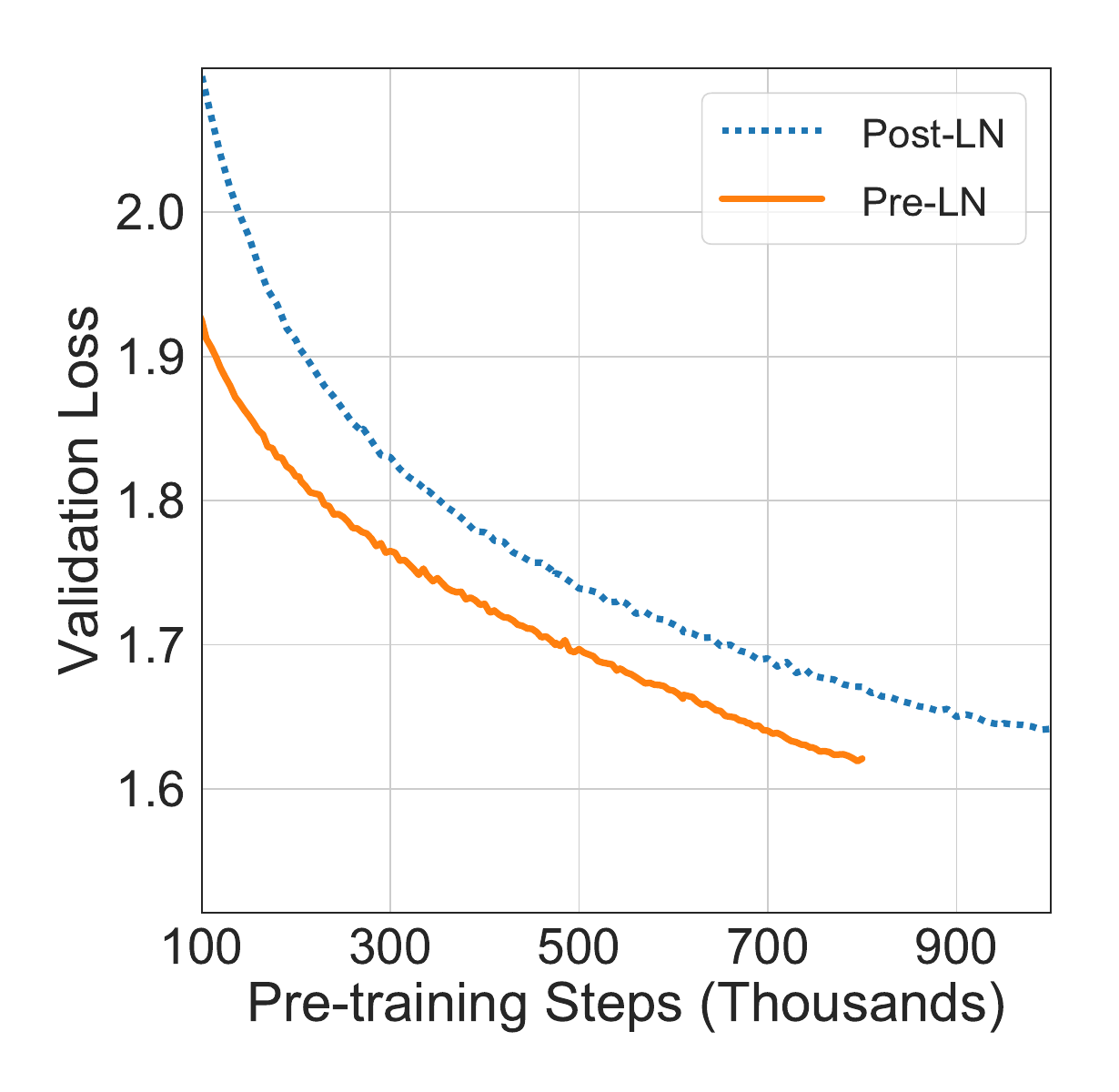}
\end{minipage}%
}%
\subfigure[Accuracy on MRPC]{
\label{fig:v1v2-glue1}
\begin{minipage}[t]{0.25\linewidth}
\centering
\includegraphics[width=\linewidth]{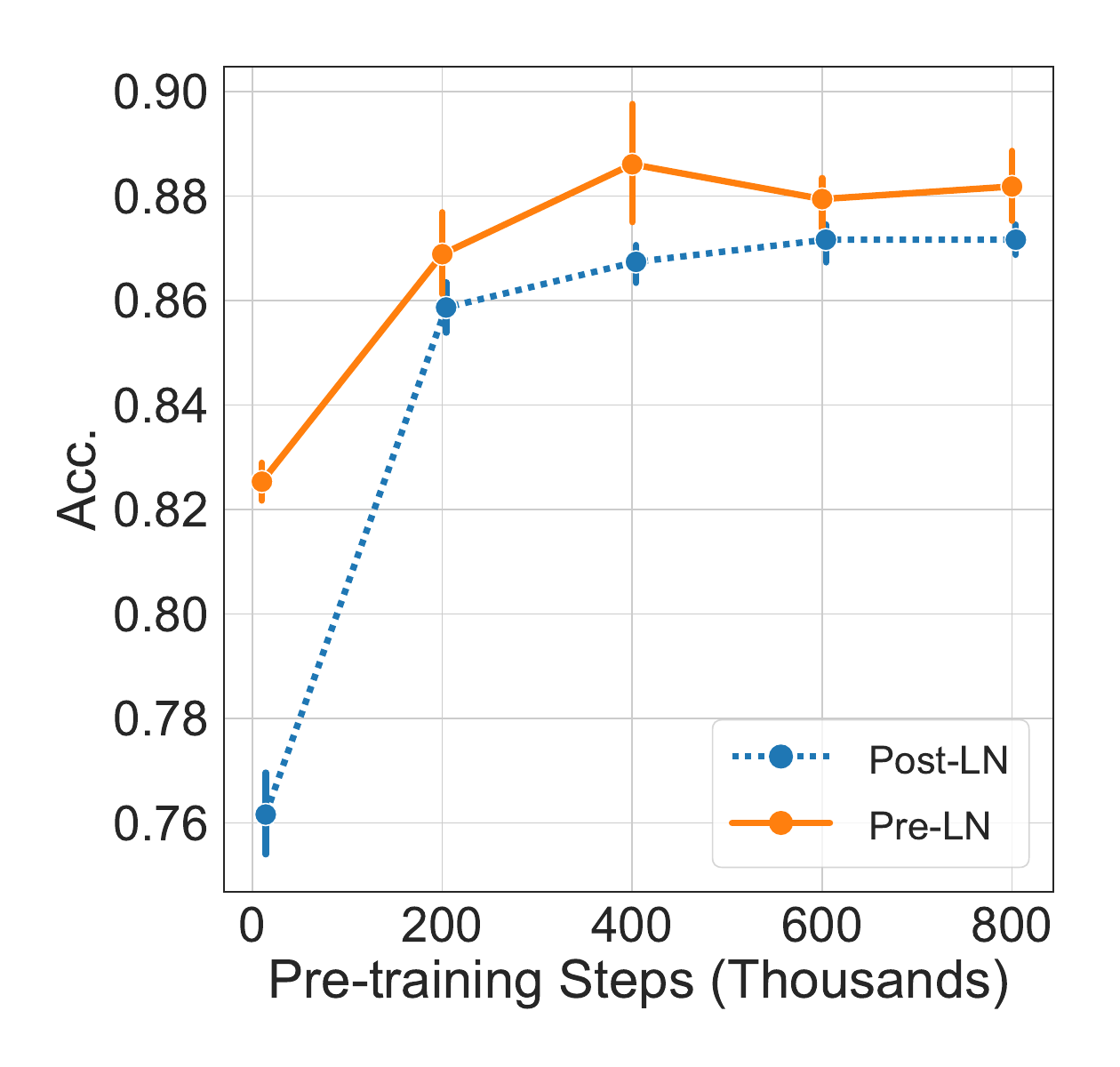}
%\caption{fig2}
\end{minipage}%
}%
\subfigure[Accuracy on RTE]{
\label{fig:v1v2-glue2}
\begin{minipage}[t]{0.25\linewidth}
\centering
\includegraphics[width=\linewidth]{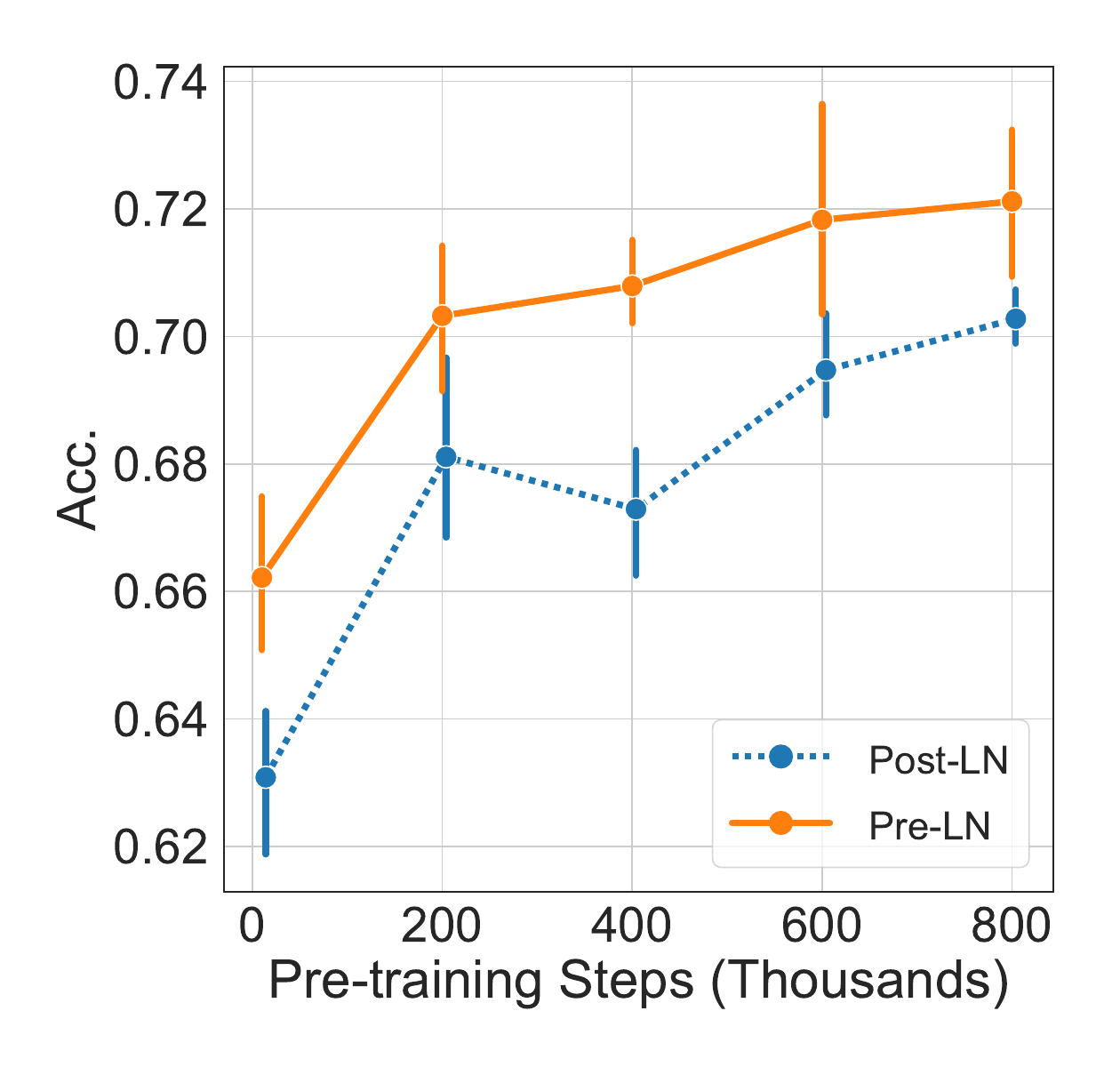}
%\caption{fig2}
\end{minipage}
}%
\centering
\caption{Performances of the models on  unsupervised pre-training (BERT) and downstream tasks}
\end{figure*}

We use \texttt{base} model configuration in our experiments. Similar to the translation task, we train the Pre-LN BERT without the warm-up stage and compare it with the Post-LN BERT. We follow the same hyper-parameter configuration in \citet{devlin2018bert} to train the Post-LN BERT using 10k warm-up steps with $\text{lr}_{max}=1e^{-4}$. For the Pre-LN BERT, we use linear learning rate decay starting from $3e^{-4}$  without the warm-up stage. We have tried to use a larger learning rate (such as $3e^{-4}$) for the Post-LN BERT but found the optimization diverged. 

\subsection{Experiment Results}

\paragraph{Machine Translation}
We record the model checkpoints for every epoch during training and calculate the validation loss and BLEU score.  The performance of the models at different checkpoints are plotted in Figure \ref{fig:v1v2-iwslt-loss} - \ref{fig:v1v2-wmt-bleu}. 

First, as we can see from the figure, the learning rate warm-up stage is not critical anymore for training the Pre-LN Transformer and the performance of the learned model is competitive. For example, on the IWSLT14 De-En task, the BLEU score and validation loss of the Pre-LN Transformer can achieve around 34 and 4, which are comparable with the performance of the Post-LN Transformer. 

Second, the Pre-LN Transformer converges faster than the Post-LN Transformer using the same $\text{lr}_{max}$. On the IWSLT14 De-En task, the 9-th checkpoint of the Pre-LN Transformer achieves nearly the same performance (validation loss/BLEU score) as 15-th checkpoint of the Post-LN Transformer. Similar observations can be found in the WMT14 En-De task. %The first model checkpoint of the Pre-LN Transformer can achieve a BLEU score near 20. As a comparison, the BLEU score of the first checkpoint of the Post-LN Transformer is less than 10.

Third, compared with RAdam, we find that the change of the position of layer normalization ``dominates'' the change of the optimizer. According to our experiments on the IWSLT14 De-En task, we can see that although RAdam trains the Post-LN Transformer well without the warm-up stage, it has little difference with Adam when training the Pre-LN Transformer.  

\paragraph{Unsupervised Pre-training (BERT)}

We record validation loss of the model checkpoints and plot them in Figure \ref{fig:v1v2-bert}. Similar to the machine translation tasks, the learning rate warm-up stage can be removed for the Pre-LN model. The Pre-LN model can be trained faster. For example, the Post-LN model achieves 1.69 validation loss at 500k updates while the Pre-LN model achieves similar validation loss at 700k updates, which suggests there is a 40\% speed-up rate. Note that $T_{warmup}$ (10k) is far less than the acceleration (200k) which suggests the Pre-LN Transformer is easier to optimize using larger learning rates. We also evaluate different model checkpoints on the downstream task MRPC and RTE (more details can be found in the supplementary material). The experiments results are plotted in Figure \ref{fig:v1v2-glue1} and \ref{fig:v1v2-glue2}. We can see that the Pre-LN model also converges faster on the downstream tasks.

As a summary, all the experiments on different tasks show that training the Pre-LN Transformer does not rely on the learning rate warm-up stage and can be trained much faster than the Post-LN Transformer.

\section{Conclusion and Future Work}
In this paper, we study why the learning rate warm-up stage is important in training the Transformer and show that the location of layer normalization matters. We show that in the original Transformer, which locates the layer normalization outside the residual blocks, the expected gradients of the parameters near the output layer are large at initialization. This leads to an unstable training when using a large learning rate. We further show that the Transformer which locates the layer normalization inside the residual blocks, can be trained without the warm-up stage and converges much faster. In the future, we will investigate other strategies of positioning the layer normalization and understand the optimization of Transformer from a theoretical perspective. 

\bibliography{example_paper}
\bibliographystyle{icml2020}

%%%%%%%%%%%%%%%%%%%%%%%%%%%%%%%%%%%%%%%%%%%%%%%%%%%%%%%%%%%%%%%%%%%%%%%%%%%%%%%
%%%%%%%%%%%%%%%%%%%%%%%%%%%%%%%%%%%%%%%%%%%%%%%%%%%%%%%%%%%%%%%%%%%%%%%%%%%%%%%

\appendix

\section{Experimental Settings}

% \subsection{Experiment on Section 3}
\subsection{Machine Translation}
% \paragraph{Experimental Settings (Section 3)} 
\paragraph{Experiment on Section 3}
The training/validation/test sets of the IWSLT14 German-to-English (De-En) task contain about 153K/7K/7K sentence pairs, respectively. We use a vocabulary of 10K tokens based on a joint source and target byte pair encoding (BPE) \citep{sennrich2015neural}. All of our experiments use a Transformer architecture with a 6-layer encoder and 6-layer decoder. The size of embedding is set to 512, the size of hidden nodes in attention sub-layer and position-wise feed-forward network sub-layer are set to 512 and 1024, and the number of heads is set to 4. Label smoothed cross entropy is used as the objective function by setting $\epsilon = 0.1$ \citep{szegedy2016rethinking}, and we apply dropout with a ratio 0.1. The batch size is set to be 4096 tokens. When we decode translation results from the model during inference, we set beam size as 5 and the length penalty as 1.2.

\paragraph{Experiment on Section 4}
The configuration of IWLST14 De-En task is the same as in Section 3\footnote{The Pre-LN Transformer can get state-of-the-art performance (35.5 test BLEU) on the IWSLT14 DE-EN task by setting initial learning rate to be $7.5e^{-4}$ and decaying it at the 8000 update steps followed by the inverse square root learning rate scheduler. The dropout is set to be 0.3, attention dropout is set to be 0.1. The batch size is set to be 8192.}.
For the WMT14 En-De task, we replicate the setup of \citep{vaswani2017attention}, which consists of about 4.5M training parallel sentence pairs, and uses a 37K vocabulary based on a joint source and target BPE. Newstest2013 is used as the validation set, and Newstest2014 is used as the test set. One of the basic configurations of the Transformer architecture is the \texttt{base} setting, which consists of a 6-layer encoder and 6-layer decoder. The size of the hidden nodes and embeddings are set to 512. The number of heads is 8. Label smoothed cross entropy is used as the objective function by setting $\epsilon= 0.1$. The batch size is set to be 8192 tokens per GPU on 16 NVIDIA Tesla P40 GPUs.

\subsection{Unsupervised Pretraining}

We follow \citet{devlin2018bert} to use English Wikipedia corpus and BookCorpus for the pre-training. As the dataset BookCorpus \citep{moviebook} is no longer freely distributed. We follow the suggestions from \citet{devlin2018bert} to crawl and collect BookCorpus\footnote{https://www.smashwords.com} on our own. The concatenation of two datasets includes roughly 3.4B words in total, which is comparable with the data corpus used in \citet{devlin2018bert}. We first segment documents into sentences with Spacy\footnote{https://spacy.io}; Then, we normalize, lower-case, and tokenize texts using Moses \citep{koehn2007moses} and apply BPE\citep{BPE}. We randomly split documents into one training set and one validation set. The training-validation ratio for pre-training is 199:1. All experiments are conducted on 32 NVIDIA Tesla P40 GPUs.

The base model in \citet{devlin2018bert} consists of 12 Transformer layers. The size of hidden nodes and embeddings are set to 768, and the number of heads is set to 12. 

\subsection{GLUE Dataset}

\paragraph{MRPC} The Microsoft Research Paraphrase Corpus \citep{dolan2005automatically} is a corpus of sentence pairs automatically extracted from online news sources, with human annotations for whether the sentences in the pair are semantically equivalent, and the task is to predict the equivalence. The performance is evaluated by the accuracy.

\paragraph{RTE} The Recognizing Textual Entailment (RTE) datasets come from a series of annual textual entailment challenges \citep{bentivogli2009fifth}. The task is to predict whether sentences in a sentence pair are entailment. The performance is evaluated by the accuracy.

\paragraph{Fine-tuning on GLUE tasks} We use the validation set for evaluation. To fine-tune the models, following \citet{devlin2018bert, liu2019roberta}, we search the optimization hyper-parameters in a search space including different batch sizes (16/32), learning rates ($1e^{-5}$ - $1e^{-4}$) and number of epochs (3-8). We find that the validation accuracy are sensitive to random seeds, so we repeat fine-tuning on each task for 6 times using different random seeds and compute the 95\% confidence interval of validation accuracy.

\section{Proof of Lemma 1}
\begin{proof}
Denote $X=(X_1,X_2,...,X_d)$ in which $X_i$ are i.i.d. Gaussian random variables with distribution $N(0,\sigma^2)$. Denote $\rho_X(x)$ as the probability density function of $X_1$. Then
$\mathbb{E}(\|\text{ReLU}(X)\|_2^2)=\sum_{i=1}^d \mathbb{E}[\text{ReLU}(X_i)^2]
=\sum_{i=1}^d \mathbb{E}[\text{ReLU}(X_i)^2|X_i\geq 0]\mathbb{P}(X_i\geq 0)
=\frac{d}{2}\mathbb{E}[\text{ReLU}(X_1)^2|X_1\geq 0]=\frac{d}{2}\mathbb{E}[X_1^2|X_1\geq 0]
=\frac{d}{2}\int_{-\infty}^{+\infty} x^2 \rho_{X|X>0}(x)dx
=\frac{d}{2}\int_{0}^{+\infty} x^2 2\rho_{X}(x)dx
=\frac{1}{2}\sigma^2 d$.
\end{proof}
\section{Proof of Lemma 2}
\begin{proof}
At initialization, the layer normalization is computed as $\text{LN}(v) = \frac{v - \mu}{\sigma}$. It is easy to see that layer normalization at initialization projects any vector $v$ onto the $d-1$-sphere of radius $\sqrt{d}$ since $\|\text{LN}(v)\|_2^2=\|\frac{v - \mu}{\sigma}\|_2^2=\frac{\sum_{k=1}^d(v_{k} -\mu)^2}{\sigma^2}=d$. 

We first estimate the expected $l_2$ norm of each intermediate output $x^{post,1}_{l,i},\cdots,x^{post,5}_{l,i}$ for $l>0$. Using Xavier initialization, the elements in $W^{V,l}$ are i.i.d. Gaussian random variables sampled from $N(0,1/d)$. Since $\|x_{l,i}^{post}\|_2^2=d$ by the definition of Layer Normalization when $l>0$,  we have
\begin{align}
\mathbb{E}(\|x_{l,i}^{post,2}\|_2^2)=&\mathbb{E}(\|x_{l,i}^{post}\|_2^2)+\mathbb{E}(\|x_{l,i}^{post,1}\|_2^2)\nonumber\\
&+2\mathbb{E}(x_{l,i}^{post,1}{x_{l,i}^{post}}^{\top})\\
=&\mathbb{E}(\|x_{l,i}^{post}\|_2^2)+\mathbb{E}(\|x_{l,i}^{post,1}\|_2^2)\nonumber\\
&+\frac{2}{n}\mathbb{E}(\sum_{j=1}^{n}x_{l, j}^{post} W^{V, l}{x_{l,i}^{post}}^{\top})\\
=&\mathbb{E}(\|x_{l,i}^{post}\|_2^2)+\mathbb{E}(\|x_{l,i}^{post,1}\|_2^2)\\
=&\mathbb{E}(\|x_{l,i}^{post}\|_2^2)+\mathbb{E}(\|\frac{1}{n}\sum_{i=1}^n x_{l,i}^{post}\|_2^2)\\
\leq &2d
\end{align}
and $\mathbb{E}(\|x_{l,i}^{post,2}\|_2^2)=\mathbb{E}(\|x_{l,i}^{post}\|_2^2)+\mathbb{E}(\|x_{l,i}^{post,1}\|_2^2)=\mathbb{E}(\|x_{l,i}^{post}\|_2^2)+\mathbb{E}(\|\frac{1}{n}\sum_{i=1}^n x_{l,i}^{post}\|_2^2)\geq \mathbb{E}(\|x_{l,i}^{post}\|_2^2)=d$. 

Similarly, we have $\|x_{l,i}^{post,3}\|_2^2=d$ by the definition of Layer Normalization. Again, for the ReLU activation function, the elements in $W^{1,l}$ and $W^{2,l}$ are i.i.d. Gaussian random variables sampled from $N(0,1/d)$. According to Lemma 1, we have 
\begin{align}
\mathbb{E}(\|x_{l,i}^{post,4}\|_2^2)
=&\mathbb{E}(\|\text{ReLU}(x_{l,i}^{post,3}W^{1,l})W^{2,l}\|_2^2)\\
=&\mathbb{E}(\mathbb{E}(\mathbb{E}(\|\text{ReLU}(x_{l,i}^{post,3}W^{1,l})W^{2,l}\|_2^2\nonumber\\
&|x_{l,i}^{post,3},W^{1,l})|x_{l,i}^{post,3}))\\
=&\mathbb{E}(\mathbb{E}(\|\text{ReLU}(x_{l,i}^{post,3}W^{1,l})\|_2^2|x_{l,i}^{post,3}))\\
=&\mathbb{E}(\frac{1}{2}\|x_{l,i}^{post,3}\|_2^2)=\frac{d}{2}
\end{align}

Based on this, we can estimate the scale of $\mathbb{E}(\|x_{l,i}^{post,5}\|_2^2)$ as follows.
\begin{align}
\mathbb{E}(\|x_{l,i}^{post,5}\|_2^2)=&\mathbb{E}(\|x_{l,i}^{post,3}\|_2^2)+\mathbb{E}(\|x_{l,i}^{post,4}\|_2^2)\nonumber\\
&+2\mathbb{E}(x_{l,i}^{post,3}{x_{l,i}^{post,4}}^{\top})\\
=&\mathbb{E}(\|x_{l,i}^{post,3}\|_2^2)+\mathbb{E}(\|x_{l,i}^{post,4}\|_2^2)\nonumber\\
&+\frac{2}{n}\mathbb{E}(\sum_{j=1}^{n}\text{ReLU}(x^{post,3}_{l,j}W^{1,l}) W^{2, l}{x_{l,i}^{post,3}}^{\top})\\
=&\mathbb{E}(\|x_{l,i}^{post,3}\|_2^2)+\mathbb{E}(\|x_{l,i}^{post,4}\|_2^2)=d+\frac{d}{2}= \frac{3}{2}d
\end{align}

Using similar technique we can bound $\mathbb{E}(\|x_{l,i}^{pre}\|_2^2)$ for the Pre-LN Transformer. 
\begin{align}
\mathbb{E}(\|x_{l,i}^{pre,3}\|_2^2)=&\mathbb{E}(\|x_{l,i}^{pre}\|_2^2)+\mathbb{E}(\|x_{l,i}^{pre,2}\|_2^2)\nonumber\\
&+2\mathbb{E}(x_{l,i}^{pre,2}{x_{l,i}^{pre}}^{\top})\\
=&\mathbb{E}(\|x_{l,i}^{pre}\|_2^2)+\mathbb{E}(\|x_{l,i}^{pre,2}\|_2^2)\nonumber\\
&+\frac{2}{n}\mathbb{E}(\sum_{j=1}^{n}x_{l, j}^{pre,1} W^{V, l}{x_{l,i}^{pre}}^{\top})\\
=&\mathbb{E}(\|x_{l,i}^{pre}\|_2^2)+\mathbb{E}(\|x_{l,i}^{pre,2}\|_2^2)\\
=&\mathbb{E}(\|x_{l,i}^{pre}\|_2^2)+\mathbb{E}(\|\frac{1}{n}\sum_{i=1}^n x_{l,i}^{pre,1}\|_2^2)
\end{align}
It is easy to see that we have $\mathbb{E}(\|x_{l,i}^{pre}\|_2^2) \leq \mathbb{E}(\|x_{l,i}^{pre,3}\|_2^2) \leq \mathbb{E}(\|x_{l,i}^{pre}\|_2^2)+d$. And similar to (10)-(12),
\begin{align}
\mathbb{E}(\|x_{l+1,i}^{pre}\|_2^2)=&\mathbb{E}(\|x_{l,i}^{pre,3}\|_2^2)+\mathbb{E}(\|x_{l,i}^{pre,5}\|_2^2)\nonumber\\
&+2\mathbb{E}(x_{l,i}^{pre,3}{x_{l,i}^{pre,5}}^{\top})\\
=&\mathbb{E}(\|x_{l,i}^{pre,3}\|_2^2)+\mathbb{E}(\|x_{l,i}^{pre,5}\|_2^2)\\
=&\mathbb{E}(\|x_{l,i}^{pre,3}\|_2^2)+\frac{1}{2}d
\end{align}

Combining both, we have $\mathbb{E}(\|x_{l,i}^{pre}\|_2^2) +\frac{1}{2}d\leq\mathbb{E}(\|x_{l+1,i}^{pre}\|_2^2)\leq \mathbb{E}(\|x_{l,i}^{pre}\|_2^2)+\frac{3}{2}d$. Then we have  $(1+\frac{l}{2})d\leq\mathbb{E}(\|x_{l,i}^{pre}\|_2^2)\leq(1+\frac{3l}{2})d$ by induction.

\end{proof}

\section{Proof of Lemma 3}
The proof of Lemma 3 is based on Lemma 4.1: 
\begin{lemma}
Let $\alpha \in \mathbb{R}^d$ be a vector such that $\|\alpha\|_2=1$, then the eigenvalue of $I-\alpha^{\top}\alpha$ is either 1 or 0.
\end{lemma}
\begin{proof}
Let $\{e_1,...,e_d\}$ be unit vectors such that $e_1=\alpha$ and $e_i\bot e_j$ for all $(i,j)$. Then we have $e_1(I-\alpha^{\top}\alpha)=e_1-e_1\alpha^{\top}\alpha=e_1-\alpha=0$ and $e_i(I-\alpha^{\top}\alpha)=e_i-e_i\alpha^{\top}\alpha=e_i$ for $i\neq 1$. So $e_i$ are all the eigenvectors of $I-\alpha^{\top}\alpha$, and their corresponding eigenvalues are $(0,1,1,...,1)$. Hence we complete our proof.
\end{proof}

\begin{proof}[Proof of Lemma 3]
Denote $y=x(I-\frac{1}{d}\textbf{1}^{\top}\textbf{1})$, where $\textbf{1}=(1,1,...,1)\in \mathbb{R}^d$, then the layer normalization can be rewritten as
\begin{align}
\text{LN}(x)_i=\frac{y_i}{\sqrt{\frac{1}{d}\sum_{j=1}^d y_j^2}}
\end{align}

We explicitly calculate the Jacobian of layer normalization as
\begin{align}
\frac{\partial \text{LN}(x)_i}{\partial y_j}=&\frac{\partial }{\partial y_j}(\frac{y_i}{\sqrt{\frac{1}{d}\sum_{k=1}^n y_k^2}})\\
=&\frac{\delta_{ij}\sqrt{\frac{1}{d}\sum_{k=1}^n y_k^2}-y_i\frac{\frac{1}{d} y_j }{\sqrt{\frac{1}{d}\sum_{k=1}^n y_k^2} }}{\frac{1}{d}\sum_{k=1}^n y_k^2}\\
=&\sqrt{d}\frac{\delta_{ij}\|y\|_2^2-y_iy_j}{\|y\|_2^{\frac{3}{2}}}=\frac{\sqrt{d}}{\|y\|_2}(\delta_{ij}-\frac{y_iy_j}{\|y\|_2^2})
\end{align}
where $\delta_{ij}=1$ when $i=j$ and $\delta_{ij}=0$ when $i\neq j$. In the matrix form,
\begin{align}
    \frac{\partial \text{LN}(x)}{\partial y}=\frac{\sqrt{d}}{\|y\|_2}(I-\frac{y^{\top}y}{\|y\|_2^2})
\end{align}
and
\begin{align}
\mathbf{J}_{LN}(x)=&\frac{\partial \text{LN}(x)}{\partial x}\\
=&\frac{\partial \text{LN}(x)}{\partial y}\frac{\partial y}{\partial x}\\
=&\sqrt{d}\frac{1}{\|y\|_2}(I-\frac{y^{\top}y}{\|y\|_2^2})(I-\frac{1}{d}\textbf{1}^{\top}\textbf{1}).
\end{align}
Since the eigenvalue of the matrix $(I-\frac{y^{\top}y}{\|y\|_2^2})$ and $(I-\frac{1}{d}\textbf{1}^{\top}\textbf{1})$ are either 1 or 0 (by Lemma 4.1), we have $\|(I-\frac{y^{\top}y}{\|y\|_2^2})\|_2=\mathcal{O}(1)$ and $\|(I-\frac{1}{d}\textbf{1}^{\top}\textbf{1})\|_2=\mathcal{O}(1)$. So the spectral norm of $\mathbf{J}_{LN}(x)$ is 
\begin{align}
\|\mathbf{J}_{LN}(x)\|_2=\mathcal{O}(\frac{\sqrt{d}}{\|y\|_2})=\mathcal{O}(\frac{\sqrt{d}}{\|x\|_2})
\end{align}

\end{proof}

\section{Proof of Theorem 1}
The proof of Theorem 1 is based on Lemma 4.2: 
\begin{lemma}
Let $Y$ be a random variable that is never larger than B. Then for all $a<B$,
\begin{equation}
\operatorname{Pr}[Y\leq a]\leq\frac{\mathbb{E}[B-Y]}{B-a}
\end{equation}
\end{lemma}
\begin{proof}
Let $X=B-Y$, then $X\geq 0$ and Markov's inequality tells us that
\begin{equation}
\operatorname{Pr}[X\geq B-a]\leq \frac{\mathbb{E}[X]}{B-a}
\end{equation}
Hence
\begin{equation}
\operatorname{Pr}[Y\leq a]\leq\frac{\mathbb{E}[B-Y]}{B-a}
\end{equation}
\end{proof}
\begin{proof}[Proof of Theorem 1]
We prove Theorem 1 by estimating each element of the gradient matrix. Namely, we will analyze $\frac{\partial \tilde{\mathcal{L}}}{\partial W^{2,L}_{pq}}$ for $p,q \in \{1,...,d\}$. The loss of the post-LN Transformer can be written as
\begin{align}
\tilde{\mathcal{L}}(x_{L+1,1}^{post},...,x_{L+1,n}^{post})=\frac{1}{n}\sum_{i=1}^n \mathcal{L}(x_{L+1,i}^{post})
\end{align}

Through back propagation, for each $i\in\{1,2,...,n\}$ the gradient of $\mathcal{L}(x_{L+1,i})$ with respect to the last layer's parameter $W^{2,L}$ in the post-LN setting can be written as:

\begin{align}
\frac{\partial \mathcal{L}(x_{L+1,i}^{post})}{\partial W^{2,L}_{pq}}=&\frac{\partial \mathcal{L}(x_{L+1,i}^{post})}{\partial x_{L+1,i}^{post}}\frac{\partial x_{L+1,i}^{post}}{\partial x_{L,i}^{post,5}}\frac{\partial x_{L,i}^{post,5}}{\partial x_{L,i}^{post,4}}\frac{\partial x_{L,i}^{post,4}}{\partial W^{2,L}_{pq}} \\
=&\frac{\partial \mathcal{L}(x_{L+1,i}^{post})}{\partial x_{L+1,i}^{post}}\mathbf{J}_{LN}(x_{L,i}^{post,5}) \frac{\partial x_{L,i}^{post,4}}{\partial W^{2,L}_{pq}}\\
=&\frac{\partial \mathcal{L}(x_{L+1,i}^{post})}{\partial x_{L+1,i}^{post}}\mathbf{J}_{LN}(x_{L,i}^{post,5})(0,0,...,\nonumber\\
&[\text{ReLU}(x_{L,i}^{post,3}W^{1,L})]_p,...,0)^{\top}
\end{align}

Here $[\text{ReLU}(x_{L,i}^{post,3}W^{1,L})]_p$ means the $p$-th element of $\text{ReLU}(x_{L,i}^{post,3}W^{1,L})$. So the absolute value of $\frac{\partial \mathcal{L}(x_{L+1,i}^{post})}{\partial W^{2,L}_{pq}}$ can be bounded by

\begin{align}
|\frac{\partial \mathcal{L}(x_{L+1,i}^{post})}{\partial W^{2,L}_{pq}}|\leq &\|\frac{\partial \mathcal{L}(x_{L+1,i}^{post})}{\partial x_{L+1,i}^{post}}\|_2\|\mathbf{J}_{LN}(x_{L,i}^{post,5})\|_2\nonumber\\
&\|(0,0,...,[\text{ReLU}(x_{L,i}^{post,3}W^{1,L})]_p,...,0)^{\top}\|_2\\
=&\|\frac{\partial \mathcal{L}(x_{L+1,i}^{post})}{\partial x_{L+1,i}^{post}}\|_2\|\mathbf{J}_{LN}(x_{L,i}^{post,5})\|_2\nonumber\\
&|[\text{ReLU}(x_{L,i}^{post,3}W^{1,L})]_p|
\end{align}
which implies
\begin{align}
|\frac{\partial \mathcal{L}(x_{L+1,i}^{post})}{\partial W^{2,L}_{pq}}|^2\leq& \|\frac{\partial \mathcal{L}(x_{L+1,i}^{post})}{\partial x_{L+1,i}^{post}}\|^2_2\|\mathbf{J}_{LN}(x_{L,i}^{post,5})\|^2_2\nonumber\\
&|[\text{ReLU}(x_{L,i}^{post,3}W^{1,L})]_p|^2
\end{align}
Since all the derivatives are bounded, we have $\|\frac{\partial \mathcal{L}(x_{L+1,i}^{post})}{\partial x_{L+1,i}^{post}}\|^2_2=\mathcal{O}(1)$. So 
\begin{align}
&|\frac{\partial \mathcal{L}(x_{L+1,i}^{post})}{\partial W^{2,L}_{pq}}|^2\nonumber\\
=&\mathcal{O}(\left[\|\mathbf{J}_{LN}(x_{L,i}^{post,5})\|^2_2|[\text{ReLU}(x_{L,i}^{post,3}W^{1,L})]_p|^2\right])
\end{align}

Since $\|x_{L,i}^{post,3}\|_2^2=d$, $[x_{L,i}^{post,3}W^{1,L}]_p$ has distribution $N(0,1)$, using Chernoff bound we have
$$\operatorname{Pr}[|[x_{L,i}^{post,3}W^{1,L}]_p|\geq a_0]\leq \operatorname{exp}(-\frac{a_0^2}{2}).$$
So $$\operatorname{Pr}[\text{ReLU}([x_{L,i}^{post,3}W^{1,L}]_p)^2\geq 2\ln{100d}]\leq \frac{0.01}{d}.$$
Thus with probability at least $0.99$, for all $p=1,2,...,d$ we have $\text{ReLU}([x_{L,i}^{post,3}W^{1,L}]_p)^2\leq 2\ln{100d}$.

Since with probability $1-\delta(\epsilon)$, $\frac{|\|x^{post,5}_{L,i}\|_2^2-\mathbb{E}\|x^{post,5}_{L,i}\|_2^2|}{\mathbb{E}\|x^{post,5}_{L,i}\|_2^2}\leq \epsilon$, we have $\|x^{post,5}_{L,i}\|_2^2\leq (1+\epsilon)\mathbb{E}\|x^{post,5}_{L,i}\|_2^2$. Using Lemma 4.2, we have
\begin{align}
\operatorname{Pr}[\|x^{post,5}_{L,i}\|_2^2\leq& \alpha_0 \mathbb{E}\|x^{post,5}_{L,i}\|_2^2]\\\leq & \frac{(1+\epsilon)\mathbb{E}\|x^{post,5}_{L,i}\|_2^2-\mathbb{E}\|x^{post,5}_{L,i}\|_2^2}{(1+\epsilon-\alpha_0)\mathbb{E}\|x^{post,5}_{L,i}\|_2^2}\\
=&\frac{\epsilon}{1+\epsilon-\alpha_0}
\end{align}
for an arbitrary constant $\alpha_0>0$, which equals
\begin{align}
\operatorname{Pr}[\|x^{post,5}_{L,i}\|_2^2\geq \alpha_0 \mathbb{E}\|x^{post,5}_{L,i}\|_2^2]
\geq 1-\frac{\epsilon}{1+\epsilon-\alpha_0}
\end{align}
So according to union bound, with probability at least $0.99-\delta(\epsilon)-\frac{\epsilon}{1+\epsilon-\alpha_0}$ we have $|\frac{\partial\mathcal{L}(x_{L+1,i}^{post})}{\partial W^{2,L}_{pq}}|^2=\mathcal{O}(\left[\|\mathbf{J}_{LN}(x_{L,i}^{post,5})\|^2_2|[\text{ReLU}(x_{L,i}^{post,3}W^{1,L})]_p|^2\right])\leq \mathcal{O}(\frac{2d\ln{100d}}{\|x_{L,i}^{post,5}\|_2^2})\leq \mathcal{O}(\frac{d\ln{d}}{\alpha_0\mathbb{E}\|x_{L,i}^{post,5}\|_2^2})=\mathcal{O}(\frac{\ln{d}}{\alpha_0})$. So we have
\begin{align}
 |\frac{\partial \tilde{\mathcal{L}}}{\partial W^{2,L}_{pq}}|^2= &|\frac{1}{n}\sum_{i=1}^n \frac{\partial \mathcal{L}(x_{L+1,i}^{post})}{\partial W^{2,L}_{pq}}|^2\\
 \leq &\frac{1}{n}\sum_{i=1}^n  |\frac{\partial \mathcal{L}(x_{L+1,i}^{post})}{\partial W^{2,L}_{pq}}|^2=\mathcal{O}(\frac{\ln{d}}{\alpha_0})
\end{align}
and $$ \|\frac{\partial \tilde{\mathcal{L}}}{\partial W^{2,L}}\|_F= \sqrt{\sum_{p,q=1}^d|\frac{\partial \tilde{\mathcal{L}}}{\partial W^{2,L}_{pq}}|^2}=\mathcal{O}(\sqrt{\frac{d^2\ln{d}}{\alpha_0}})$$.

The loss of the pre-LN Transformer can be written as
\begin{align}
\tilde{\mathcal{L}}(x_{Final,1}^{pre},...,x_{Final,n}^{pre})=\frac{1}{n}\sum_{i=1}^n \mathcal{L}(x_{Final,i}^{pre})
\end{align}
Using the same technique, in the pre-LN setting the gradient of $\mathcal{L}(x_{Final,i}^{pre})$ with respect to the last layer's parameter $W^{2,L}$ can be written as
\begin{align}
\frac{\partial \mathcal{L}(x_{Final,i}^{pre})}{\partial W^{2,L}_{pq}}
=&\frac{\partial \mathcal{L}(x_{Final,i}^{pre})}{\partial x_{Final,i}^{pre}}
\frac{\partial x_{Final,i}^{pre}}{\partial x_{L+1,i}^{pre}}
\frac{\partial x_{L+1,i}^{pre}}{\partial x_{L,i}^{pre,5}}
\frac{\partial x_{L,i}^{pre,5}}{\partial W^{2,L}_{pq}}\\
=&\frac{\partial \mathcal{L}(x_{Final,i}^{pre})}{\partial x_{Final,i}^{pre}}\mathbf{J}_{LN}(x_{L+1,i}^{pre})(0,0,...,\nonumber\\
&[\text{ReLU}(x_{L,i}^{pre,4}W^{1,L})]_p,...,0)^{\top}
\end{align}
So the absolute value of each component of the gradient is bounded by
\begin{align}
|\frac{\partial \mathcal{L}(x_{Final,i}^{pre})}{\partial W^{2,L}_{pq}}|\leq &\|\frac{\partial \mathcal{L}(x_{Final,i}^{pre})}{\partial x_{Final,i}^{pre}}\|_2\|\mathbf{J}_{LN}(x_{L+1,i}^{pre})\|_2\nonumber\\
&\|(0,0,...,[\text{ReLU}(x_{L,i}^{pre,4}W^{1,L})]_p,...,0)\|_2\\
=&\|\frac{\partial \mathcal{L}(x_{Final,i}^{pre})}{\partial x_{Final,i}^{pre}}\|_2\|\mathbf{J}_{LN}(x_{L+1,i}^{pre})\|_2\nonumber\\
&|[\text{ReLU}(x_{L,i}^{pre,4}W^{1,L})]_p|
\end{align}

Since $\| x_{L,i}^{pre,4}\|_2^2=d$ and $[ x_{L,i}^{pre,4}W^{1,L}]_p$ obeys distribution $N(0,1)$, using Chernoff bound we have
$$\operatorname{Pr}[|[ x_{L,i}^{pre,4}W^{1,L}]_p|\geq a_0]\leq \operatorname{exp}(-\frac{a_0^2}{2}).$$
So $$\operatorname{Pr}[\text{ReLU}([ x_{L,i}^{pre,4}W^{1,L}]_p)^2\geq 2\ln{100d}]\leq \frac{0.01}{d}.$$
So with probability at least $0.99$, for all $p=1,2,...,d$ we have $\text{ReLU}([ x_{L,i}^{pre,4}W^{1,L}]_p)^2\leq 2\ln{100d}$.

Since with probability $1-\delta(\epsilon)$, $\frac{|\|x^{pre}_{L+1,i}\|_2^2-\mathbb{E}\|x^{pre}_{L+1,i}\|_2^2|}{\mathbb{E}\|x^{pre}_{L+1,i}\|_2^2}\leq \epsilon$, we have $\|x^{pre}_{L+1,i}\|_2^2\leq (1+\epsilon)\mathbb{E}\|x^{pre}_{L+1,i}\|_2^2$. Using Lemma 5, we have
\begin{align}
\operatorname{Pr}[\|x^{pre}_{L+1,i}\|_2^2\leq &\alpha_0 \mathbb{E}\|x^{pre}_{L+1,i}\|_2^2]\\
\leq& \frac{(1+\epsilon)\mathbb{E}\|x^{pre}_{L+1,i}\|_2^2-\mathbb{E}\|x^{pre}_{L+1,i}\|_2^2}{(1+\epsilon-\alpha_0)\mathbb{E}\|x^{pre}_{L+1,i}\|_2^2}\\
=&\frac{\epsilon}{1+\epsilon-\alpha_0}
\end{align}
which equals
\begin{align}
\operatorname{Pr}[\|x^{pre}_{L+1,i}\|_2^2\geq \alpha_0 \mathbb{E}\|x^{pre}_{L+1,i}\|_2^2]\geq 1-\frac{\epsilon}{1+\epsilon-\alpha_0}
\end{align}
According to union bound, with probability $0.99-\delta(\epsilon)-\frac{\epsilon}{1+\epsilon-\alpha_0}$ we have $|\frac{\partial \mathcal{L}(x_{Final,i}^{pre})}{\partial W^{2,L}_{pq}}|^2=\mathcal{O}(\left[\|\mathbf{J}_{LN}(x_{L+1,i}^{pre})\|^2_2|[\text{ReLU}(x_{L,i}^{pre,4}W^{1,L})]_p|^2\right])\leq \mathcal{O}(\frac{2d\ln{100d}}{\|x_{L+1,i}^{pre}\|_2^2})\leq \mathcal{O}(\frac{d\ln{d}}{\alpha_0\mathbb{E}\|x_{L+1,i}^{pre}\|_2^2})=\mathcal{O}(\frac{\ln{d}}{\alpha_0 L})$. So we have
\begin{align}
|\frac{\partial \tilde{\mathcal{L}}}{\partial W^{2,L}_{pq}}|^2=|\frac{1}{n}\sum_{i=1}^n \frac{\partial \mathcal{L}(x_{Final,i}^{pre})}{\partial W^{2,L}_{pq}}|^2=\mathcal{O}(\frac{\ln{d}}{\alpha_0 L})
\end{align}
Thus $\|\frac{\partial \tilde{\mathcal{L}}}{\partial W^{2,L}}\|_F=\sqrt{\sum_{p,q=1}^d|\frac{\partial \tilde{\mathcal{L}}}{\partial W^{2,L}_{pq}}|^2}\leq \mathcal{O}(\sqrt{\frac{d^2\ln{d}}{\alpha_0L}})$.

Take $\alpha_0=\frac{1}{10}$, we have that with probability at least $0.99-\delta(\epsilon)-\frac{\epsilon}{0.9+\epsilon}$, for the Post-LN Transformer we have $\|\frac{\partial \tilde{\mathcal{L}}}{\partial W^{2,L}}\|_F\leq \mathcal{O}(d\sqrt{\ln{d}})$ and for the Pre-LN Transformer we have $\|\frac{\partial \tilde{\mathcal{L}}}{\partial W^{2,L}}\|_F\leq \mathcal{O}(d\sqrt{\frac{\ln{d}}{L}})$
\end{proof}

\section{Extension to other layers}

For simplicity, we denote $x_l=\text{Concat}(x_{l,1},...,x_{l,n})\in \mathbb{R}^{nd}$ and $x_l^k=\text{Concat}(x_{l,1}^k,...,x_{l,n}^k)\in \mathbb{R}^{nd}$ for $k=\{1,2,3,4,5\}$. Then in the Post-LN Transformer, the gradient of the parameters in the $l$-th layer (take $W^{2,l}$ as an example) can be written as
$$\frac{\partial \tilde{\mathcal{L}}}{\partial W^{2,l}}=\frac{\partial \tilde{\mathcal{L}}}{\partial x^{post}_{L+1}}(\prod_{j=l+1}^{L}\frac{\partial x^{post}_{j+1}}{\partial x^{post}_j})\frac{\partial x^{post}_{l+1}}{\partial W^{2,l}},$$ where $$\frac{\partial x^{post}_{j+1}}{\partial x^{post}_j}=
\frac{\partial x^{post}_{j+1}}{\partial x_j^{post,5}}\frac{\partial x_{j}^{post,5}}{\partial x_j^{post,3}}\frac{\partial x_{j}^{post,3}}{\partial x_j^{post,2}}\frac{\partial x_{j}^{post,2}}{\partial x_j^{post}}.
$$
The Jacobian matrices of the Post-LN Transformer layers are:
\begin{align}
\frac{\partial x_{j+1}^{post}}{\partial x_j^{post,5}}=&\left(
  \begin{array}{ccc}
    \mathbf{J}_{LN}(x_{j,1}^{post,5}) &  &  \\
     & \ddots &  \\
     &  & \mathbf{J}_{LN}(x_{j,n}^{post,5}) \\
  \end{array}
\right)\\
\frac{\partial x_j^{post,5}}{\partial x_j^{post,3}}=&\left(
  \begin{array}{ccc}
    I &  &  \\
     & \ddots &  \\
     &  &  I\\
  \end{array}
\right)+
\left(
  \begin{array}{ccc}
    W^{2,j} &  &  \\
     & \ddots &  \\
     &  & W^{2,j} \\
  \end{array}
\right)\nonumber\\
&
\left(
  \begin{array}{ccc}
    \mathbf{J}^{j}_1 &  &  \\
     & \ddots &  \\
     &  & \mathbf{J}^{j}_n \\
  \end{array}
\right)
\left(
  \begin{array}{ccc}
    W^{1,l} &  &  \\
     & \ddots &  \\
     &  & W^{1,l} \\
  \end{array}
\right)
\end{align}
where 
\begin{dmath*}
\mathbf{J}^{j}_i = \operatorname{diag}\left(\sigma^{\prime}\left( x_{j,i}^{post,3}\left(\mathbf{w}_{1}^{1,j}\right)^{\top}\right), 
...,\sigma^{\prime}\left( x_{j,i}^{post,3}\left(\mathbf{w}_{d}^{1,j}\right)^{\top}\right)\right) \in \mathbb{R}^{d \times d}
\end{dmath*}

\begin{gather}
\frac{\partial x_j^{post,3}}{\partial x_j^{post,2}}=\left(
  \begin{array}{ccc}
    \mathbf{J}_{LN}(x_{j,1}^{post,2}) &  &  \\
     & \ddots &  \\
     &  & \mathbf{J}_{LN}(x_{j,n}^{post,2}) \\
  \end{array}
\right)\\
\frac{\partial x_j^{post,2}}{\partial x_j^{post}}=\left(
  \begin{array}{ccc}
    I &  &  \\
     & \ddots &  \\
     &  &  I\\
  \end{array}
\right)+
\left(
  \begin{array}{ccc}
    \frac{1}{n}W^{V,j} & \cdots & \frac{1}{n}W^{V,j} \\
    \vdots & \ddots & \vdots \\
    \frac{1}{n}W^{V,j} & \cdots & \frac{1}{n}W^{V,j} \\
  \end{array}
\right)
\end{gather}
% where $z_{j,i}=x_{j,i}^{post,2}(I-\frac{1}{d}\textbf{1}^{\top}\textbf{1})$.

Using H\"{o}lder's inequality, we have
\begin{align}
&\mathbb{E}\|\frac{\partial x_{j+1}^{post}}{\partial x_j^{post}}\|_2\nonumber\\
\leq&
\mathbb{E}\left[\|\frac{\partial x_{j+1}^{post}}{\partial x_j^{post,5}}\|_2\|\frac{\partial x_{j}^{post,5}}{\partial x_j^{post,3}}\|_2\|\frac{\partial x_{j}^{post,3}}{\partial x_j^{post,2}}\|_2\|\frac{\partial x_{j}^{post,2}}{\partial x_j^{post}}\|_2\right]\\
\leq&\sqrt{\mathbb{E}\left[\|\frac{\partial x_{j+1}}{\partial x_j^{post,5}}\|_2^2\right]\mathbb{E}\left[\|\frac{\partial x_{j}^{post,5}}{\partial x_j^{post,3}}\|_2^2\|\frac{\partial x_{j}^{post,3}}{\partial x_j^{post,2}}\|_2^2\|\frac{\partial x_{j}^{post,2}}{\partial x_j^{post}}\|_2^2\right]}
\end{align}
Since $\frac{\partial x_{j+1}}{\partial x_j^{post,5}}=diag(\mathbf{J}_{LN}(x_{j,1}^{post,5}),...,\mathbf{J}_{LN}(x_{j,n}^{post,5}))$, we have $\sqrt{\mathbb{E}\left[\|\frac{\partial x_{j+1}^{post}}{\partial x_j^{post,5}}\|_2^2\right]} \approx \sqrt{\mathbb{E}\frac{d}{\|x_{j,1}^{post,5}\|^2_2}} \approx \sqrt{\frac{2}{3}}$ when $\|x_{j,1}^{post,5}\|^2_2$ concentrates around its expectation $\mathbb{E}\|x_{j,1}^{post,5}\|^2_2$ which equals $\frac{3}{2}d$ according to Lemma 2. Therefore, when we estimate the norm of $\frac{\partial \tilde{\mathcal{L}}}{\partial W^{2,l}}$ for post-LN transformer, there exists a term $\mathcal{O}(\frac{2}{3}^{(L-l)/2})$, which exponentially decreases as $l$ goes smaller. Similarly, in the pre-LN Transformer, the gradient can be written as
$$\frac{\partial \tilde{\mathcal{L}}}{\partial W^{2,l}}=\frac{\partial \tilde{\mathcal{L}}}{\partial x^{pre}_{Final}}\frac{\partial x^{pre}_{Final}}{\partial x^{pre}_{L+1}}(\prod_{j=l+1}^{L}\frac{\partial x^{pre}_{j+1}}{\partial x^{pre}_j})\frac{\partial x^{pre}_{l+1}}{\partial W^{V,l}},$$ where $$ \frac{\partial x^{pre}_{j+1}}{\partial x^{pre}_j}=
\frac{\partial x^{pre}_{j+1}}{\partial x_j^{pre,3}}\frac{\partial x_{j}^{pre,3}}{\partial x_j^{pre}}.
$$

The Jacobian matrices of the Pre-LN Transformer layers are:

\begin{align}
\frac{\partial x_{j+1}^{pre}}{\partial x_j^{pre,3}}=&\left(
  \begin{array}{ccc}
    I &  &  \\
     & \ddots &  \\
     &  &  I\\
  \end{array}
\right)+
\left(
  \begin{array}{ccc}
    W^{2,j} &  &  \\
     & \ddots &  \\
     &  & W^{2,j} \\
  \end{array}
\right)\nonumber\\
&
\left(
  \begin{array}{ccc}
    \mathbf{J}^{\left(h^{\prime}\right)}_1 &  &  \\
     & \ddots &  \\
     &  & \mathbf{J}^{\left(h^{\prime}\right)}_n \\
  \end{array}
\right)
\left(
  \begin{array}{ccc}
    W^{1,j} &  &  \\
     & \ddots &  \\
     &  & W^{1,j} \\
  \end{array}
\right)\nonumber\\
&
\left(
  \begin{array}{ccc}
    \mathbf{J}_{LN}(x^{pre,3}_{j,1}) &  &  \\
     & \ddots &  \\
     &  & \mathbf{J}_{LN}(x^{pre,3}_{j,n}) \\
  \end{array}
\right)
\end{align}
\begin{align}
\frac{\partial x_{j}^{pre,3}}{\partial x_j^{pre}}=&\left(
  \begin{array}{ccc}
    I &  &  \\
     & \ddots &  \\
     &  &  I\\
  \end{array}
\right)+
\left(
  \begin{array}{ccc}
    \frac{1}{n}W^{V,j} & \cdots & \frac{1}{n}W^{V,j} \\
    \vdots & \ddots & \vdots \\
    \frac{1}{n}W^{V,j} & \cdots & \frac{1}{n}W^{V,j} \\
  \end{array}
\right)\nonumber\\
&
\left(
  \begin{array}{ccc}
    \mathbf{J}_{LN}(x_{j,1}^{pre}) &  &  \\
     & \ddots &  \\
     &  & \mathbf{J}_{LN}(x_{j,n}^{pre}) \\
  \end{array}
\right)
\end{align}
If $l$ is  sufficiently large, the norm of $\mathbf{J}_{LN}(x_{j,i}^{pre})$ and $\mathbf{J}_{LN}(x^{pre,3}_{j,i})$ are very small (of order $\mathcal{O}(\frac{1}{\sqrt{j}})$) as $j$ is between $l+1$ and $L$, which means the eigenvalues of matrix $\frac{\partial x_{j+1}^{pre}}{\partial x_j^{pre,3}}$ and $\frac{\partial x_{j}^{pre,3}}{\partial x_j^{pre}}$ are close to 1. Then we can see that $\mathbb{E}\|\frac{\partial x^{pre}_{j+1}}{\partial x_j^{pre,3}}\|_2$ and $\mathbb{E}\|\frac{\partial x_{j}^{pre,3}}{\partial x_j^{pre}}\|_2$ are nearly 1, and the norm of $\frac{\partial \tilde{\mathcal{L}}}{\partial W^{2,l}}$ for pre-LN transformer is independent of $l$ when $l$ is large. 
\begin{figure*}[htbp]
\centering

\label{fig:v1-small-update}
\begin{minipage}[t]{1.0\linewidth}
\centering
\includegraphics[width=\linewidth]{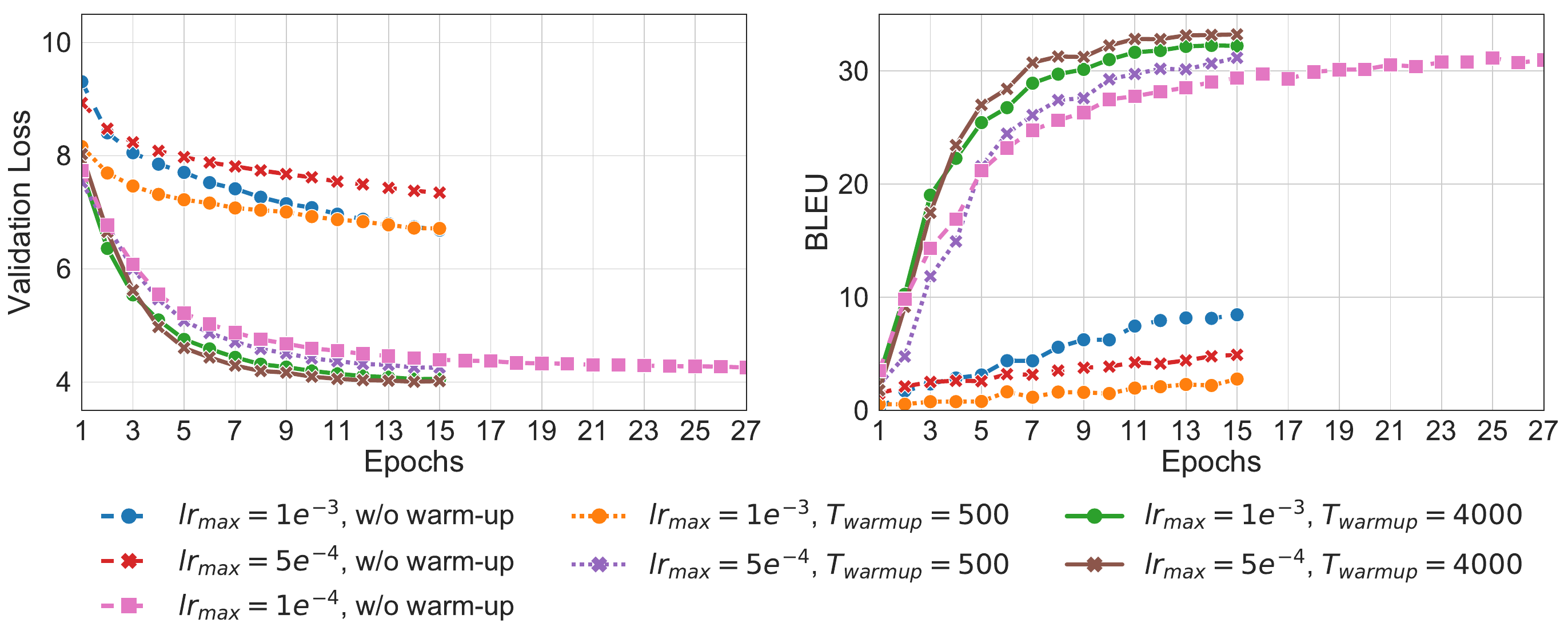}
\end{minipage}
\caption{Performances of the models on the IWSLT14 De-En task.}
\end{figure*}

\section{Examples of $(\epsilon,\delta)$-bounded random variables}
In this section we give an example of $(\epsilon,\delta)$-bounded random variable. This example comes from Example 2.5 in \citep{wainwright2019high} and we give a short description below.

If $Z=(Z_1,...,Z_n)$ is a Gaussian vector with distribution $N(0,I_n)$, then $Y=\|Z\|^2_2=\sum_{k=1}^n Z_k^2$ has distribution $\chi^2_n$. And $\mathbb{E}Y=\sum_{k=1}^n \mathbb{E}Z_k^2=n$

A random variable $X$ with mean $\mu=\mathbb{E}[X]$ is called \emph{sub-exponential} if there are non-negative parameters $(\nu,\alpha)$ such that $\mathbb{E}[\exp(\lambda(X-\mu))]\leq \exp(\frac{\nu^2\lambda^2}{2})$ for all $|\lambda|<\frac{1}{\alpha}$. The next proposition comes from Proposition 2.2 in \citep{wainwright2019high}.
\begin{proposition}[Sub-exponential tail bound]
Suppose that $X$ is sub-exponential with parameters $(\nu,\alpha)$. Then 
\begin{equation}
\mathbb{P}[X-\mu \geq t] \leq\left\{\begin{array}{ll}{\exp(-\frac{t^{2}}{2 \nu^{2}})} & {\text { if } 0 \leq t \leq \frac{\nu^{2}}{\alpha}, \text { and }} \\ {\exp(-\frac{t}{2 \alpha})} & {\text { for } t>\frac{\nu^{2}}{\alpha}}\end{array}\right.
\end{equation}
and from Example 2.5 in \citep{wainwright2019high}, the $\chi^2$ variable $Y$ is sub-exponential with parameters $(\nu, \alpha)=(2 \sqrt{n}, 4)$. So we can derive the one-sided bound 
\begin{equation}
    \mathbb{P}\left[Y-n\geq n\epsilon\right] \leq \exp(-n \epsilon^{2} / 8), \quad \text { for all } \epsilon \in(0,1)
\end{equation}
So $Y$ is $(\epsilon,\delta)$-bounded with $\epsilon \in(0,1)$ and $\delta=\exp(-n\epsilon^2/8)$.

\end{proposition}

\section{Small learning rate experiment}

Theoretically, we find that the gradients of the parameters near the output layers are very large for the Post-LN Transformer and suggest using large learning rates to those parameters makes the training unstable. To verify whether using small-step updates mitigates the issue, 
we use a very small but fixed learning rate and check whether it can optimize the Post-LN Transformer (without the learning rate warm-up step) to a certain extent. In detail, we use a fixed learning rate of $1e^{-4}$ at the beginning of the optimization, which is much smaller than the $\text{lr}_{max}= 1e^{-3}$ in the paper. Please note that as the learning rates during training are small, the training converges slowly, and this setting is not very practical in real large-scale tasks. We plot the validation curve together with other baseline approaches in Figure 6. We can see from the figure, the validation loss (pink curve) is around 4.3 in 27 epochs. This loss is much lower than that of the Post-LN Transformer trained using a large learning rate (blue curve). But it is still worse than the SOTA performance (green curve).

\end{document}